\documentclass[10pt,journal,compsoc]{IEEEtran}
\usepackage{subfigure}
\usepackage{tikz}
\tikzset{global scale/.style={
    scale=#1,
    every node/.append style={scale=#1}
  }
}
\usepackage{romannum}
\usepackage{amsmath}
\usepackage{amssymb}
\usepackage{amsthm}
\usepackage[ruled,vlined,linesnumbered]{algorithm2e}
\usepackage{cancel}
\usepackage{multirow}
\usepackage{stfloats}
\usepackage{booktabs}
\usepackage{makecell}
\usepackage{cite}
\usepackage{xr}
\usepackage{color}
\usepackage{xcolor}
\usepackage[colorlinks = true,
            linkcolor = blue,
            urlcolor  = blue,
            citecolor = blue,
            anchorcolor = blue]{hyperref}
\newcommand{\changreflcolor}[1]{\hypersetup{linkcolor=#1}}   

\usepackage{enumitem}
\usepackage{floatrow} 
\usepackage{mathtools,amssymb,lipsum}

\usepackage{cuted}
\usepackage{tikz}
\usepackage{textcomp}
\usepackage{hyperref}
\usepackage{lipsum}

\newcommand\copyrighttext{%
  \footnotesize \textcopyright 2022 IEEE. Personal use of this material is permitted.
  Permission from IEEE must be obtained for all other uses, in any current or future
  media, including reprinting/republishing this material for advertising or promotional
  purposes, creating new collective works, for resale or redistribution to servers or
  lists, or reuse of any copyrighted component of this work in other works.
  DOI: 10.1109/TNET.2022.3223367}
\newcommand\copyrightnotice{%
\begin{tikzpicture}[remember picture,overlay]
\node[anchor=south,yshift=10pt] at (current page.south) {\fbox{\parbox{\dimexpr\textwidth-\fboxsep-\fboxrule\relax}{\copyrighttext}}};
\end{tikzpicture}%
}

\newcommand*{\addFileDependency}[1]{
  \typeout{(#1)}
  \@addtofilelist{#1}
  \IfFileExists{#1}{}{\typeout{No file #1.}}
}
\makeatother

\ifCLASSINFOpdf
\else
\fi

\begin{document}

\newtheorem{definition}{Definition}
\newtheorem{lemma}{Lemma}
\newtheorem{proposition}{Proposition}
\newtheorem{theorem}{Theorem}
\newtheorem{corollary}{Corollary}
\newtheorem*{conjecture}{Conjecture}
\newtheorem{example}{Example}
\newtheorem{remark}{Remark}
\newtheorem{observation}{Observation}

\newenvironment{proofsketch}{%
  \renewcommand{\proofname}{Proof Sketch}\proof}{\endproof}

\newcommand{\RNum}[1]{\lowercase\expandafter{\romannumeral #1\relax}}
\newcommand{\RNumUpper}[1]{\uppercase\expandafter{\romannumeral #1\relax}}
 
\title{Social Sourcing: Incorporating Social Networks Into Crowdsourcing Contest Design}
 
\author{Qi~Shi,~\IEEEmembership{Student Member,~IEEE,}
        Dong~Hao,~\IEEEmembership{Member,~IEEE}
\IEEEcompsocitemizethanks{\IEEEcompsocthanksitem The two authors have equal contribution. 
\IEEEcompsocthanksitem Affiliation:  University of Electronic Science and Technology of China.\protect
\IEEEcompsocthanksitem E-mails: mailofshiqi@gmail.com; haodongpost@gmail.com
}
}
 
\markboth{IEEE/ACM Transactions on Networking}%
{Shell \MakeLowercase{\textit{et al.}}: Bare Demo of IEEEtran.cls for Computer Society Journals}
 
\IEEEtitleabstractindextext{%
\begin{abstract}
In a crowdsourcing contest, a principal holding a task posts it to a crowd. People in the crowd then compete with each other to win the rewards. Although in real life, a crowd is usually networked and people influence each other via social ties, existing crowdsourcing contest theories do not aim to answer how interpersonal relationships influence people's incentives and behaviors and thereby affect the crowdsourcing performance. In this work, we novelly take people's social 
ties as a key factor in the modeling and designing of agents' incentives in crowdsourcing contests. 
We establish two contest mechanisms by which the principal can impel the agents to invite their neighbors to contribute to the task. The first mechanism has a symmetric Bayesian Nash equilibrium, and it is very simple for agents to play and easy for the principal to predict the contest performance. The second mechanism has an asymmetric Bayesian Nash equilibrium, and agents' behaviors in equilibrium show a vast diversity which is strongly related to their social relations. The Bayesian Nash equilibrium analysis of these new mechanisms reveals that, besides agents' intrinsic abilities, the social relations among them also play a central role in decision-making. Moreover, we design an effective algorithm to automatically compute the Bayesian Nash equilibrium of the invitation crowdsourcing contest and further adapt it to a large graph dataset. Both theoretical and empirical results show that the new invitation crowdsourcing contests can substantially enlarge the number of participants, whereby the principal can obtain significantly better solutions without a large advertisement expenditure.
\end{abstract}

}

\maketitle

\copyrightnotice
 
\IEEEdisplaynontitleabstractindextext
 
\IEEEpeerreviewmaketitle

\IEEEraisesectionheading{\section{Introduction}\label{sec:introduction}}

\subsection{Background}
\IEEEPARstart{C}{rowdsourcing}  is the action of an organisation outsourcing its task to an undefined and generally large crowd of people using an open call for participation \cite{howe2006rise}.
Contest has proven to be a very successful approach in many platforms for crowdsourcing. A contest is a social and economic interaction where a group of agents exert costly and irretrievable effort in order to win prizes and the principal needs to make decision of judiciously awarding prizes among  these agents based on their performances. It can be modeled as a game in which agents strategize on how to contribute effort and compete against each other. The applications of contest design have varied in human and multi-agent systems, ranging from political campaigns, sports, advertising, school admission, labour contract, R\&D competition to crowdsourcing. The development of contest design theory has been specially advanced over recent years in artificial intelligence, fueled by the need for various internet online services applications.

In a crowdsourcing contest, a principal posts a task on a platform and announces a monetary reward that he is willing to pay for a winning solution. Participants submit solutions to the platform, and the principal chooses the best solution and awards the prize. For example, in the world's largest competitive software development portal TopCoder.com, software development phases are realised through contests in which the crowd of developers compete with each other for the best innovation. Popular projects on GitHub can also provide a bounty to elicit better code from users. The Netflix Prize contest awarded one million dollars to an outstanding team that created the most accurate algorithm for predicting user ratings. Social recruitment on LinkedIn or Recruiter.com can also be seen as a crowdsourcing contest where the recruiters try to acquire staff with the highest ability from a crowd.

\subsection{Challenges}
The crowdsourcing contest is essentially about how to elicit efforts from a crowd of people and how to allocate rewards among them. In real life, a crowd is usually networked, and people influence other people via social ties. However, existing crowdsourcing contest applications and theories do not formally consider \emph{how interpersonal relationships influence people's incentives and decisions, and thereby affect the overall performance of the crowdsourcing}. Without considering the social relations, incentive design in crowdsourcing contests is inadequate. This is because an individual can exert her influence on a crowdsourcing not only by contributing her own solution, but also via her relations with other people. For instance, it is commonly seen that people share  information with their friends about job opportunities, investment opportunities, challenges, contests or money-making tasks.
However, the state-of-the-art crowdsourcing contests do not aim to answer why people share information with others. They do not treat people's social influences as a key factor in behavioral reasoning and incentive design. By the current contest models, the principal can only convene some people via her own effort. Talents with high ability may not know the task and will not contribute, which results in a limited quality of the crowdsourcing.



\subsection{Our Contribution}
\emph{We propose a novel crowdsourcing contest model considering that people may invite their friends to contribute to a task}. Our model is essentially a {\textit{Bayesian game on graph}} where each node's behavior has two dimensions: whether to contribute her ability and which neighbors to invite. This is a non-typical high-dimensional contest model and is different from traditional contests where the principal can only call together agents via her effort.


Based on this Bayesian game model,  we establish two invitation mechanisms for crowdsourcing contests by which a principal can impel agents to invite their friends to contribute. Each mechanism has three steps: (\romannumeral 1) agents register in the contest, and they can invite friends to register; (\romannumeral 2) agents contribute their solutions; (\romannumeral 3) the principal chooses the winners and allocates the rewards. The two mechanisms have different rules and different information exposure. The first one publishes the dynamically changing number of registered agents. It has a symmetric Bayesian Nash equilibrium (BNE) where every agent's strategy is simple to play: when the number of registered agents surpasses a specific value, invite all neighbors but not contribute; otherwise, do not invite anyone but contribute.
The second mechanism allows agents to 
observe the social graph. This empowers them to make deliberative decisions taking into consideration of social relations. This mechanism admits asymmetric BNE, where every agent has her own threshold   to decide whether to contribute, and all agents prefer to invite all  neighbors. We also devise a recursive algorithm for the computation of the complex asymmetric BNE.


We  also conduct extensive crowdsourcing contest experiments on a Facebook network dataset. Both theoretical and experimental results show that,
with the invitation model and the mechanisms, the crowdsourcing contest principal can convene as many people as possible but without a large amount of advertisement expenditure. Thus the hidden valuable people in the social network get the chance to participate even though they are not known directly by the principal. The principal obtains significantly better solutions without a large extra payment on propagation.

\subsection{Related Works}
The works on contest design are  traced back to \cite{tullock1967welfare} and \cite{krueger1974political}, which studied  rent-seeking problems.
Another pioneering  work \cite{lazear1981rank} showed that rank-order tournaments could serve as an efficient scheme to elicit effort from strategic agents.
Optimal contest design, which aims at eliciting better submissions, dates back to \cite{glazer1988optimal} and has recently been investigated in a wide variety of settings in both economics and artificial intelligence.

For instance, some works focused on the effects of contest success functions (e.g.\cite{dasgupta1998designing, fu2020optimal}), some studied the prize structures in contests (e.g.\cite{che2003optimal, olszewski2020performance}), and some tried to find the optimal contest under different output-skill-effort functions and maximizing goals  (e.g.\cite{ghosh2012implementing, ghosh2016optimal}). 
There is also a valuable topic about multi-stage contests \cite{fu2012optimal}.
With the concept of crowdsourcing burgeoning with the development of internet, many works embed contest research into the crowdsourcing scenario, such as prize structures \cite{archak2009optimal, cavallo2012efficient, luo2015incentive, chawla2019optimal}, multiple contests on crowdsourcing platforms \cite{dipalantino2009crowdsourcing}, and crowdsourcing effectiveness under specific schemes \cite{gao2012quality, cavallo2013winner}.
A comprehensive survey of crowdsourcing contests is in \cite{brabham2013crowdsourcing} and \cite{segev2020crowdsourcing}.

A recent branch of works on contest theory proposed the simple contest mechanisms \cite{ghosh2016optimal,Levy2017contest,Sarne2017effective}. The simple contest captures a common scenario that an agent's submission quality is intrinsic and cannot be adjusted. Thus agents in a simple contest only need to decide whether to participate or not. For example, in a job interview, all the candidates would do their best and the performances only depend on their abilities. 
The previous works on simple contest all deal with a fixed set of agents. There are several models of simple contest having been discussed, including models of single or multiple winners \cite{ghosh2016optimal}, certain or uncertain performance evaluation \cite{habani2019contest}, parallel or sequential competing order \cite{levy2019selective} and so on.
In this work, we also follow the basic ideas of simple contest design but design a game and mechanisms in a more realistic setting: the domain of agents is expanded from a fixed set to all agents in the social network, where the principal can only reach a portion of the whole crowd.


There has also been works about query propagation in social networks. In this problem,  each node is offered with a reward by its parent, and offers a smaller reward to its children for the propagation of a query.
The first model formulating the query propagation on a random network was given by \cite{kleinberg2005query}. Later works such as \cite{arcaute2007threshold, dikshit2009truthful, jin2011discovering} extended this work to more general conditions. The famous 2009 DARPA Network Challenge witnessed the outstanding achievements of this series of works \cite{pickard2011time}. Some up-to-date results (e.g.\cite{wang2018bcinet, shen2019multi})  extended the scenario from query to contest, where the propagation of crowdsourcing through social ties plays an important role. 
Regarding the worker recruitment crowdsourcing, a recent work \cite{wang2021socialrecruiter} 
divided the task duration into multiple cycles and dynamically updated the propagating and completing rewards. 
Besides the query incentive network problems, researchers also proposed auction-based models for crowdsourcing \cite{jin2018incentive}.
There have been some very recent works on auction mechanism design in social networks \cite{li2017mechanism, zhao2018selling, kawasaki2020strategy, li2022diffusion}. As far as we know, all these works concentrate on introducing social networks into typical auction problems, which are quite different from contest design. Meanwhile,  the above existing works for crowdsourcing did not provide a general game-theoretical analysis for the incentive mechanism design. They also didn't identify the Bayesian Nash equilibrium of crowdsourcing considering network structure's impact \cite{guo2021emerging}.


%



%

\section{Invitation Contest Model}
This section defines the invitation contest as a Bayesian game and characterises the Bayesian Nash equilibrium. Then, combining with graph theory, we represent several structural notions which will be used throughout the paper.  
\subsection{Bayesian Game for Invitation Contest}
Consider a crowd of agents $\mathbf{N}=\{1,\dots,n\}$. The social network underlying these agents is a graph $\mathcal{G=(\mathbf{N},\mathbf{E})}$ where the edge set $\mathbf{E}$ is the social relationships among agents in $\mathbf{N}$. Denote agent $i$'s neighbor set by $E_i$. Two nodes can reach each other directly if there is an edge linking them, but the structure of the whole graph is unobservable by any agent.  There is a principal $p\in \mathbf{N}$, who has a task and wants to obtain high-quality solutions. Principal $p$ holds a contest to find high-ability agents and to acquire solutions. Without other agents' help, $p$ can only call together a small portion of agents in $\mathbf{N}$, which she can directly influence (either by advertisement or notification). 


Any agent can participate in the contest by contributing her own effort, or by inviting her neighbors, or both, or neither. 
An agent contributing her effort will incur a cost $c>0$ while making an invitation is costless. An agent could be rewarded for providing a good solution to the task or for inviting a high-ability agent.  
An agent's ability $q_i$ is a value from a continuous interval $ [l, u]$ where $l \geq 0$ is the lower bound of agents' abilities and $u$ is the upper bound. The higher the value of an agent's ability is, the more proficient the agent is. Assume the abilities of all agents are i.i.d. with probability density function (pdf) $f(q)$ and cumulative distribution function (cdf) $F(q)$.


Each agent's ability is her \emph{private information} and can not be observed by other agents. That is, the game is with incomplete information. Therefore, we model the contest in social networks as a Bayesian game  \cite{harsanyi1967games}, which is the conventional model for games with incomplete information. 
  \begin{definition}
A Bayesian game is a tuple $\left<\mathbf{N},\mathbf{q},f,\mathbf{s},\boldsymbol{\pi}\right>$ where
$\mathbf{N}$ is the set of  agents; $\mathbf{q}=(q_1,\cdots,q_n)$ is the profile of  agents' abilities;
$f:\mathbf{q} \to \mathbb R$ is the joint distribution of agents' abilities; 
$\mathbf{s}=(s_1, \cdots, s_n)$ is the profile of agents' strategies, where $s_i:q_i \rightarrow (a_i,e_i)$ is a mapping from ability $q_i$ to contribution action $a_i$ and invitation action $e_i$;
$\boldsymbol{\pi}=(\pi_1, \cdots, \pi_n)$ is the profile of agents' utility functions, where $i$' utility function $\pi_i : \mathbf{q}\times\mathbf{s}\to\mathbb R $.

\end{definition}




Especially, agent $i$'s strategy
$s_i:q_i \rightarrow (a_i,e_i)$ shows the two-dimensional decision: one is the contribution action $a_i\in \{q_i, 0\}$, meaning that if an agent contributes, she contributes with her full ability. Such a contribution setting has a wide range of contest applications \cite{ghosh2016optimal, priel2018tractable}.  The other dimension is the invitation action $e_i\in\mathcal{P}(E_i)$, i.e., which friends (neighbors in $\mathcal{G}$) to invite. $\mathcal{P}(E_i)$ is the powerset of $i$'s neighbor set $E_i$. 
Note that an agent being invited does not necessarily mean she must contribute. Also, an agent inviting some friends does not necessarily mean she does not contribute. 

\subsection{Equilibrium Concept}
In classic games with complete information, it is assumed that an agent perfectly knows which kind of opponents she is playing with. The Bayesian game we construct for a contest in social networks is more realistic because the uncertainty of agents' types (ability in contests) is captured. In Bayesian games, 
an agent's ability cannot be observed by any other agent. Still, all agents' abilities are drawn from the prior joint probability distribution $f=p(q_1,\cdots,q_n)$,
which is common knowledge. 


Using the Bayes rule, conditioning on knowing her own ability and the prior $f=p(q_1,\cdots,q_n)$, an agent can compute the conditional distribution $p({\bf q}_{-i}|q_i)$, which quantifies the joint probability of her opponents' ability profile  ${\bf q}_{-i}=(q_i,\dots,q_{i-1},q_{i+1},\dots,q_{n})$. With the conditional distribution known, agent $i$ can compute the expected payoff for each of her strategy $s_i(q_i)=(a_i,e_i)$. Denote the strategy profile of all agents except $i$ by ${\bf s}_{-i}=(s_1,\dots,s_{i-1},s_{i+1},\dots,s_{n})$, we can have the solution concept Bayesian Nash equilibrium. 
\begin{definition}
The strategy profile ${\bf s}^*=(s^*_1, \dots, s^*_n)$ for the contest in social networks is a Bayesian Nash equilibrium (BNE) if for every agents $i\in \mathbf{N}$, her two-dimensional strategy is
\begin{equation}\label{eq:BNE definition}
    s^*_i\left(q_i\right)=\mathop{\arg\max}\limits_{s'_i}\sum_{{\bf q}_{-i}}
    p({\bf q}_{-i}|q_i)
    \pi_i\big({s'_i}(q_i),{{\bf s}^*_{-i}({\bf q}_{-i})}\big).
\end{equation}
\end{definition}
Strategy profile ${\bf{s^*}}$ is a BNE if no agent $i$ can increase her \textit{expected utility} by unilaterally deviating from this profile.  Recall that $p({\bf q}_{-i}|q_i)$ can be computed by  $i$ via Bayes rule. 

\begin{remark}
    In Bayesian games,  each agent's objective  is to maximize the expected utility, i.e., the sum of many different $\pi_i$ weighted by $p({\bf q}_{-i}|q_i)$. Agents do not intend to maximize any  specific payoff $\pi_i$. Under some  profile ${\bf q}_{-i}$, the value of $\pi_i$ could be low, and even negative. However, this does not necessarily mean the expected utility is  low.
\end{remark}

Finding the BNE for the Bayesian game according to Eq.(\ref{eq:BNE definition}) is like solving a linear combination of many games with complete information. However, it is worth noting that, the difficulty of this work lies in the two-dimensional strategy of each agent, where one dimension is her personal attribute $a_i$, and the other is her social attribute $e_i$. These two dimensions of actions are entangled when agents make decisions and are also affected by other agents' contributions and invitations. The invitation contest is essentially a complex  \textit{Bayesian game on graph}, which has not been formalized and tackled before. These make our problem much more challenging than solving a typical Bayesian game. 





\subsection{Invitation Graph and Ranking Tree}\label{subsection:ranking tree}
With agents' invitations, the task information propagates in the social network $\mathcal G$. 
All the invited agents and all inviting relations form a subgraph $\mathcal H=(\mathbf{U} \cup \{p\},\mathbf{E}') \subseteq \mathcal G$, where the invited agent set $\mathbf{U}\subseteq \mathbf{N}$ and the inviting relation set $\mathbf{E}' \subseteq \mathbf{E}$. We call $\mathcal{H}$ the \emph{invitation graph}.

Now we identify some other terminologies, which will be crucial for the contest mechanism design and equilibrium analysis in the following sections. The relation \textit{beat} depicts the ability ranking of agents, while the relation  \textit{lead}  captures the structural ranking of the nodes in $\mathcal H$.


\begin{definition} (Beat)\label{df:beat}
    Agent $i$ beats agent $j$, if $i$ and $j$ both contribute and the contribution qualities $q_i > q_j$, or if $i$ contributes but $j$ does not.
\end{definition}
With agents' invitations taken into consideration, besides the ability, an agent's position in the network also plays a crucial role. We use the following notion to depict the relative position of any two agents in the invitation graph.
\begin{definition} (Lead)\label{df:critical referrer}
Agent $i$ leads agent $j$ (formally, $i\succ j$) , if $j$ cannot know the task when $i$ doesn't invite any neighbor.
\end{definition}
That is, $i$ is a cut node in $\mathcal{H}$ and if she does not propagate the task information, then all the  routes from the principal $p$ to $j$ are broken, resulting in that $j$ cannot know the task and will not contribute. Denote  by $\mathbf{D}_i$ the set of all agents led by $i$, and by $\mathbf{C}_i$ the set of all agents leading $i$.
Note that both $\mathbf{D}_i$ and $\mathbf{C}_i$ exclude $i$ herself.  When $i$ leads $j$, we have $i \in \mathbf{C}_j$, $j \in \mathbf{D}_i$ and $\mathbf{D}_j \subset \mathbf{D}_i$.

The leading/led relations among all agents in any given invitation graph $\mathcal H$ can be described by a  \emph{ranking tree} $\mathcal T$. In a ranking tree, the root is the principal $p$, and other nodes are invited agents. The nodes in $\mathbf{C}_i$ forms the branch from $p$ to $i$.  Every node in this branch is $i$'s leader. All the nodes in the sub-tree rooted at $i$ form the set $\mathbf{D}_i$. Note that by definition, in $\mathcal T$, the branch $\mathbf{C}_i$ is a set of cut nodes, but it is not necessary that these cut nodes are adjacent in $\mathcal H$.

Given the social network of a crowd, there could be different invitation graphs depending on different strategy profiles of all agents. But with a specific invitation graph, the ranking tree is unique. We show an example of the above notions in Fig. \ref{fig:graph and the domination tree} and Example \ref{ex:fig1}.

{\begin{figure}[thb]
    \centering
    \begin{subfigure}
        \centering
		\begin{tikzpicture}[thick, global scale=0.75]
		    \node (1) at (0.0,0.0)[circle,draw,minimum size=6mm]{};
    		\node (2) at (-1.0,-0.8)[circle,draw,minimum size=6mm]{}
    		edge[-] (1);
    		\node (3) at (1.0,-0.8)[circle,draw,minimum size=6mm]{}
    		edge[-] (1);
    		\node (4) at (0.0,-1.6)[circle,draw,minimum size=6mm]{}
    		edge[-] (3);
    		\node (5) at (2.0,-1.6)[circle,draw,minimum size=6mm]{}
    		edge[-] (3);
    		\node (6) at (1.0,-1.6)[circle,draw,minimum size=6mm]{}
    		edge[-] (4)
    		edge[-] (5);
    		\node (7) at (1.0,-2.4)[circle,draw,minimum size=6mm]{}
    		edge[-] (6);
    		\node (16) at (-1.0, -1.6)[circle,draw,minimum size=6mm]{}
    		edge[-] (2);
    		\node (18) at (2.0, -2.4)[circle,draw,minimum size=6mm]{}
    		edge[-] (5);
    		\node (8) at (0.0,0.0){\large $p$};
    		\node (9) at (-1.0,-0.8){\large $a$};
    		\node (10) at (1.0,-0.8){\large $b$};
    		\node (11) at (0.0,-1.6){\large $c$};
    		\node (12) at (2.0,-1.6){\large $d$};
    		\node (13) at (1.0,-1.6){\large $e$};
    		\node (14) at (1.0,-2.4){\large $f$};
    		\node (17) at (-1.0, -1.6){\large $g$};
    		\node (19) at (2.0, -2.4){\large $h$};
    		\node (15) at (0.5,-3.2){\large (a) $\mathcal{G}$ };
		\end{tikzpicture}
    \end{subfigure}
    \begin{subfigure}
        \centering
		\begin{tikzpicture}[thick,global scale=0.75]
		    \node (1) at (0.0,0.0)[circle,draw,minimum size=6mm]{};
    		\node (2) at (-1.0,-0.8)[circle,draw,minimum size=6mm]{}
    		edge[-] (1);
    		\node (3) at (1.0,-0.8)[circle,draw,minimum size=6mm]{}
    		edge[-] (1);
    		\node (4) at (0.0,-1.6)[circle,draw,minimum size=6mm]{}
    		edge[-] (3);
    		\node (5) at (2.0,-1.6)[circle,draw,minimum size=6mm]{}
    		edge[-] (3);
    		\node (6) at (1.0,-1.6)[circle,draw,minimum size=6mm]{}
    		edge[-] (4)
    		edge[-] (5);
    		\node (7) at (1.0,-2.4)[circle,draw,minimum size=6mm]{}
    		edge[-] (6);
    		\node (8) at (0.0,0.0){\large $p$};
    		\node (9) at (-1.0,-0.8){\large $a$};
    		\node (10) at (1.0,-0.8){\large $b$};
    		\node (11) at (0.0,-1.6){\large $c$};
    		\node (12) at (2.0,-1.6){\large $d$};
    		\node (13) at (1.0,-1.6){\large $e$};
    		\node (14) at (1.0,-2.4){\large $f$};
    		\node (15) at (0.5,-3.2){\large (b) $\mathcal{H}$ };
		\end{tikzpicture}
    \end{subfigure}
    \begin{subfigure}
        \centering
		\begin{tikzpicture}[thick, global scale=0.75]
		    \node (1) at (0.0,0.0)[circle,draw,minimum size=6mm]{};
    		\node (2) at (-1.0,-0.8)[circle,draw,minimum size=6mm]{}
    		edge[-] (1);
    		\node (3) at (1.0,-0.8)[circle,draw,minimum size=6mm]{}
    		edge[-] (1);
    		\node (4) at (0.0,-1.6)[circle,draw,minimum size=6mm]{}
    		edge[-] (3);
    		\node (5) at (2.0,-1.6)[circle,draw,minimum size=6mm]{}
    		edge[-] (3);
    		\node (6) at (1.0,-1.6)[circle,draw,minimum size=6mm]{}
    		edge[-] (3);
    		\node (7) at (1.0,-2.4)[circle,draw,minimum size=6mm]{}
    		edge[-] (6);
    		\node (8) at (0.0,0.0){\large $p$};
    		\node (9) at (-1.0,-0.8){\large $a$};
    		\node (10) at (1.0,-0.8){\large $b$};
    		\node (11) at (0.0,-1.6){\large $c$};
    		\node (12) at (2.0,-1.6){\large $d$};
    		\node (13) at (1.0,-1.6){\large $e$};
    		\node (14) at (1.0,-2.4){\large $f$};
    		\node (15) at (0.5,-3.2){\large (c) $\mathcal{T}$ };
	    \end{tikzpicture}
    \end{subfigure}

    \caption{(a) $\mathcal{G}$: the  social network of a crowd. (b) $\mathcal{H}$: an invitation graph emerging on the social network. (c) $\mathcal{T}$: the ranking tree re-constructed based on the invitation graph. }
    \label{fig:graph and the domination tree}
\end{figure}
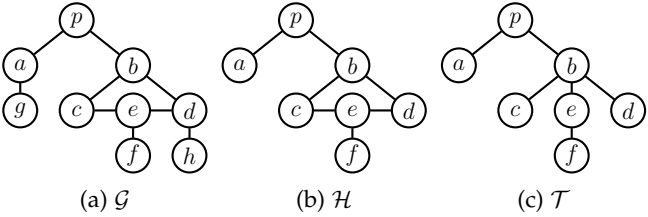 }
\begin{example} \label{ex:fig1}
(1) Let Fig.\ref{fig:graph and the domination tree}(a) be the social network where agents $a$ and $d$ invite no one but other agents invite all their neighbors. (2) Then the invitation graph is in Fig.\ref{fig:graph and the domination tree}(b). From this invitation graph, we have $\mathbf U=\{a,b,c,d,e,f\}$, $\mathbf{C}_b = \emptyset$, $\mathbf{D}_b = \{c,d,e,f\}$, $\mathbf{C}_e = \{b\}$ and $\mathbf{D}_e = \{f\}$. (3) The  ranking tree generated from this invitation graph is shown in Fig.\ref{fig:graph and the domination tree}(c).
\end{example}
Specifically, $b$ does not directly link to  $e$ in $\mathcal{H}$ while $b$ leads $e$ in  $\mathcal{T}$. This is because $\mathcal{T}$ is not simply a sub-graph extracted from $\mathcal{H}$; rather it is re-constructed based on the leading  relation in Definition \ref{df:critical referrer}. After removing $b$, there will be no route from $p$ to $e$, thus $b$ leads $e$. Comparatively, albeit $c$ directly links to $e$ in $\mathcal{H}$, only removing $c$ will not break the information flow from $p$ to $e$. Thus $c$ does not lead $e$ in $\mathcal{T}$.

\section{Simple Invitation Mechanism}
In this section, we design the first invitation contest mechanism in which the number of participants is real-time public information, and the award phase has a naive and straightforward rule. We prove that this mechanism admits a simple Bayesian Nash equilibrium: each agent either contributes her full ability or invites all her friends, but not both, nor somewhere in between. The equilibrium strategy of this mechanism is very simple for agents to implement. Therefore, we call it \emph{simple invitation mechanism} (SIM).

\subsection{Contest Mechanism Characterization}
Simple invitation mechanism (SIM) is designed to have the following phases.
\begin{enumerate}
 \item \textbf{Registration phase}: Agents register in the contest, and they can invite friends to register. The principal publishes the total number of registered agents in real time.
\item \textbf{Competition phase}: Registered agents finally decide whether to contribute their abilities or not.
\item \textbf{Award phase}: The agent contributing the highest ability wins prize $\mathcal{M}$, and each leading node of the winner is awarded a tiny invitation bounty $\epsilon$. 
\end{enumerate}
In the registration phase, the principal publishes the registered agent number in real time. This is the information exposed to all the agents, which will be considered when they strategize on their contribution actions and invitation actions. During the registration and competition phases, agents play a Bayesian game, and the utilities are realized in the award phase. 
Note that in the registration phase, agents can also think about whether to contribute, but their final decisions are made in the competition phase.
We show the allocation rule of the award phase in Algorithm \ref{alg:BC-RM}.
\begin{algorithm}[ht ]
        \KwIn{invited agents $\mathbf{U}$, invitation graph $\mathcal H$, agents' contribution qualities $q_{i}, \forall i \in \mathbf{U}$}
        \KwOut{winner $w$, $w$'s leading agent set $\mathbf{C}_w$}

        initialize $w = nil$, $\mathbf{C}_w = \emptyset$; \\
        \If{there is at least one contribution}{
            compute $w \leftarrow \mathop{\arg\max}\limits_{i \in \mathbf{U}} q_i$, tie-break if needed; \\
            \For{$i \in \mathbf{U}$}{
                \If{$i \neq w$ and $i$ leads $w$}{
                    $\mathbf{C}_w \leftarrow \mathbf{C}_w + \{i\}$; \\
                }
            }
        }
        return $w, \mathbf{C}_w$.
	\caption{Allocation rule in award phase of SIM}
	\label{alg:BC-RM}
\end{algorithm}


Now we analyze agents' incentives and behaviors under SIM, which will further help us to characterize the BNE of SIM. 
We will show that every agent has a \emph{common threshold} $r$ in a given invitation graph. If an agent $i$'s ability $q_i$ is above this threshold, then she prefers to contribute; otherwise, she prefers not to contribute. We will further prove that, in the BNE of SIM, every agent $i$'s strategy (i.e.,  $s_i:q_i \rightarrow (a_i,e_i)$) has a very  simple form: if the number of registered
agents is less than her critical value, her best decision is  contributing ($a_i=q_i$) but inviting no neighbor ($e_i=\emptyset$); otherwise, her best decision is inviting all neighbors ($e_i=E_i$) but not contributing ($a_i=0$).

\subsection{Ability Threshold for Contributing}

Consider an agent $i$ in a given invitation graph with ability value $q_i$. Given any strategy of another invited agent $j$, the probability that $i$ beats $j$ when contributing is $P(a_j < q_i)$. This value is non-decreasing and continuous in $q_i$. Besides, denote the probability that $i$ does not win but is a leading node of winner $w$ when contributing by $P(i\in {\bf C}_w)$.

If agent $i$ contributes, her  expected utility $\pi_i$
is the sum of the expected prize of being the winner and the expected invitation bounty of leading the winner minus the cost for contributing, that is
\begin{equation} \notag
    \pi_i(a_i=q_i) = \mathcal{M} \cdot \prod_{j \in \mathbf{U}_{-i}} P\big(a_j < q_i\big) + \epsilon \cdot P(i\in {\bf C}_w) - c,
\end{equation}\\
where $\mathbf{U}_{-i}$ is the set of all invited agents other than $i$. 

If agent $i$ does not   contribute, her expected utility can  only be from the expected invitation bounty of leading the winner, that is
\begin{equation} \notag
    \pi_i(a_i=0) = \epsilon \cdot P'(i\in {\bf C}_w),
\end{equation}
where $P'(i\in {\bf C}_w)$ is the probability $i$ is a leading node of the winner when $i$ does not contribute. $P'(i\in {\bf C}_w)$ is not identical to $P(i\in {\bf C}_w)$.

Agent $i$ will contribute iff $ \pi_i(a_i=q_i)> \pi_i(a_i=0)$. Moreover, since we assume the invitation bounty $\epsilon$ is very small comparing with $M$, then approximately, $i$ will contribute when
\begin{equation} \label{eq:derive threshold in SIM}
    \prod_{j \in \mathbf{U}_{-i}} P\big(a_j < q_i\big)  > \frac{c}{\mathcal{M}}.
\end{equation}
Since $P(a_j<q_i)$ is non-decreasing and continuous in $q_i$, the left side of Eq.\eqref{eq:derive threshold in SIM} is also non-decreasing and continuous in $q_i$. Therefore, there must be a threshold $r_i \in [l, u]$ and $i$ will contribute iff $q_i > r_i$. Such a strategy using this  threshold is a \emph{best response} for every agent against any given strategy profile of her opponents. 
\begin{proposition} (Ability Threshold)\label{def:threshold strategy}
In the simple invitation mechanism, every agent contributes iff her ability $q_i > r_i$, where $r_i$ is her ability threshold.
If all agents share an identical $r_i=r$, this is a common threshold.
\end{proposition}

In what follows, we will show that under SIM, there exists a common threshold $r$ that, in equilibrium, every agent with ability $q_i > r$ will contribute, and every agent with ability $q_i \le r$ does not contribute. We will derive this common threshold $r$ in an explicit form.

\subsection{Beating and Winning Probabilities}

The key to prove the existence of a common threshold in the BNE is a typical game-theoretical reasoning scheme that,
for an arbitrary 
agent $i$, given that all other agents' 
thresholds are $r$, then $i$ also has this  threshold $r$. Note that we do not assume $r_i=r$ as a premise, and we even do not assume $i$ has any threshold. But eventually, we will prove that, $i$ indeed has a threshold and it equals $r$. We first calculate the probability that agent $i$ beats the other agents.

Consider any other agent $j$ who has threshold $r$. 
According to Definition \ref{df:beat}, for any $q_i \in [l,u]$, there are two cases where $i$ beats $j$: one is they both contribute but $i$ is better, i.e., $r < q_j < q_i$; the other is $i$ contributes but $j$ doesn't, i.e., $q_j \leq r$.
Note that we do not assume $i$ has threshold $r$, so it is not necessary that $r < q_i$.
To sum up, if $i$ beats $j$, it means $q_j \le \max \{q_i, r\}$. Then the cumulative probability that $i$ beats $j$ is
\begin{equation} \label{eq         :F_qj}
P\big(q_j\le\max \{q_i, r\}\big) = F_{q_j}\big(\max \{q_i, r\}\big).
\end{equation}
Under i.i.d. assumption, multiplying this cumulative probability over all $j$, we directly have the following lemma and corollary.
\begin{lemma}\label{lemma:winning_prob}
The probability that agent $i$ wins under SIM is $
\prod_{j \in \mathbf{U}_{-i}} F_{q_j}(\max \{q_i, r\}),
$ where $q_i$ is her ability, $r$ is the ability threshold, and $\mathbf {U}_{-i}$ is the set of all invited agents other than $i$. 
\end{lemma}
\begin{corollary}\label{corollary:winning_prob}
Under the i.i.d. assumption of agents' abilities,  the probability for $i$ to win is $F(\max \{q_i, r\})^{|\mathbf U|-1}$, and the probability for $i$ not to win is $1 - F(\max \{q_i, r\})^{|\mathbf{U}|-1} $. 
\end{corollary}
Note that, value of this winning probability still depends on the number of $j$ in $\mathbf{U}$ (i.e. the size of $\mathbf U$), and agents' invitations may change this value. But on a specific invitation graph obtained in the registration phase of SIM, we do know the size of $\mathbf U$. Next we will derive the formula of $r$.

\subsection{Common Threshold in Symmetric Equilibrium}
Now we analyze agent $i$'s contribution behavior under SIM.
There are two cases of $i$'s contribution.
The first case is $i$ contribute with her ability $q_i$. In this case, agent $i$'s utility is composed of two parts: one is the expected benefit from winning prize $\mathcal{M}$ and the other is that from leading the winner.
According to Corollary \ref{corollary:winning_prob}, the former happens with a probability 
$$ p^w_i(q_i,e_i) =F\big(\max \{q_i, r\}\big)^{|\mathbf{U}|-1}, $$ 
and the later happens with a probability $$p^l_i(q_i,e_i)=\left(1 - F\big(\max \{q_i, r\}\big)^{|\mathbf{U}|-1}\right)\cdot\frac{|\mathbf{D}_i|}{|\mathbf{U}_{-i} |}. $$
Recall that $\mathbf{D}_i$ is the set of all agents led by $i$ and the fraction ${|\mathbf{D}_i|}/{|\mathbf{U}_{-i} |}$ denotes the probability that a random agent is led by $i$. Thus the product is the probability that $i$ leads the winner. To sum up, the expected utility for $i$ to contribute is 
\begin{equation} \label{eq:pi_contribute}
\pi_i(q_i,e_i)=\mathcal{M} \cdot p^w_i(q_i,e_i)-c + \epsilon \cdot p^l_i(q_i,e_i).
\end{equation}




The other case is $i$ doesn't contribute.
She may still benefit from leading the winner. The probability that there exists a winner is $1-F(r)^{|\mathbf{U}|-1} $. This is just the probability that there exists at least one agent contributing. Then the probability that the winner is led by $i$ is 
$$p^l_i(0,e_i)=\left(1-F(r)^{|\mathbf{U}|-1} \right)\cdot\frac{|\mathbf{D}_i|}{|\mathbf{U}_{-i} |}. $$
Thus the expected utility for $i$ not to contribute is:
\begin{equation} \label{eq:pi_no_contribute}
\pi_i(0,e_i)=\epsilon \cdot p^l_i(0,e_i).
\end{equation}

To ensure that agent $i$ has incentive to contribute, the difference between Eq.\eqref{eq:pi_contribute} and Eq.\eqref{eq:pi_no_contribute} should be non-negative, that is 
$\mathcal{M} \cdot p^w_i(q_i,e_i)-c + \epsilon \cdot \left[p^l_i(q_i,e_i)-p^l_i(0,e_i)\right]\ge 0$.
Since comparing with the prize $\mathcal{M}$, the tiny bounty is very small, then approximately, the binding constraint for $i$ to contribute is $p^w_i(q_i,e_i)= \frac{c}{\mathcal{M}}$, which finally results in the common threshold as
\begin{equation} \label{eq:threshold_final}
r= F^{-1}\left(\sqrt[\leftroot{-3}\uproot{10}|\mathbf{U}|-1]{\frac{c}{\mathcal{M}}}\right). 
\end{equation}

The threshold in Eq.(\ref{eq:threshold_final}) is derived by the binding condition. For an arbitrary agent $i$, the difference between her contributing utility in Eq.(\ref{eq:pi_contribute}) and non-contributing utility in Eq.(\ref{eq:pi_no_contribute}) is positive if and only if her ability $q_i >r$.

With this threshold existing, 
for any agent, her best strategy is: to contribute if her ability is above this threshold, and not to contribute otherwise. 
\begin{proposition}[Common Threshold]\label{lemma:SIM_threshold_strategy_is_BNE}
under SIM,  there is a symmetric Bayesian Nash equilibrium where every agent contributes iff her ability is above the  common ability threshold Eq.(\ref{eq:threshold_final}).
\end{proposition}

\subsection{Invitation in Symmetric Equilibrium}
Proposition \ref{lemma:SIM_threshold_strategy_is_BNE} shows that the ability threshold $r$ governs agents' contributing decisions in the competition phase. Now the remaining question is: for each agent $i$, how to decide whom to invite in the registration phase? Recall that under SIM, the total number of registered agents is public knowledge, thus next we will show how every agent's strategy $s_i$ is depending on this registration number.





    

If $q_i > r$, on the one hand, according to Proposition \ref{lemma:SIM_threshold_strategy_is_BNE}, agent $i$ will contribute. The utility for contributing in Eq.(\ref{eq:pi_contribute}) is decreasing in the size of $\mathbf U$. This means when $q_i > r$, agent $i$ is better not to invite anyone. On the other hand, by Eq.(\ref{eq:threshold_final}), $q_i > r
$ is equivalent to $q_i >F^{-1}(\sqrt[\leftroot{-3}\uproot{3}|\mathbf{U}|-1]{\frac{c}{\mathcal{M}}})$. Inversing this function we have 
$
|\mathbf{U}| < log_{F(q_i)} \frac{c}{\mathcal{M}} + 1.
$
This means when the total number of invited agents is less than a critical number $t(q_i) = log_{F(q_i)} \frac{c}{\mathcal{M}} + 1$, agent $i$ is better not to invite anyone. Note that, although $t(q_i)$ is increasing in $q_i$, for any fixed agent $i$ with ability $q_i$, the critical number $t(q_i) $ is a constant.

When $q_i < r$, agent $i$ will not contribute. Now the only way she can obtain some benefit is being a leader of the winner. 
In this case, $i$ needs to invite as many neighbors as possible, because inviting more neighbors can only increase the chance of being a leader of the winner. 
This is intuitive. Consider that if agent $i$ does not contribute, according to Eq.\eqref{eq:pi_no_contribute}, $i$'s expeted utility is $$
    \pi_i(0,e_i)= \epsilon \cdot p^l_i(0,e_i) = \epsilon \cdot \left(1-F(r)^{|\mathbf{U}|-1} \right)\cdot\frac{|\mathbf{D}_i|}{|\mathbf{U}_{-i}|}.
$$\\
Replacing $r$ in this formula by Eq.\eqref{eq:threshold_final}, we have $
     \pi_i(0,e_i)= \epsilon \cdot \left(1- \frac{c}{\mathcal{M}} \right)\cdot\frac{|\mathbf{D}_i|}{|\mathbf{U}_{-i}|},\notag
$
where $\mathbf{D}_i \subseteq \mathbf{U}_{-i}$, and $|\mathbf{U}_{-i}| - |\mathbf{D}_i|$ is a value independent of $i$'s invitation behavior.
This indicates that $\pi_i(0,e_i)$ is monotonically increasing in $|\mathbf{D}_i|$. Moreover, we know that $|\mathbf{D}_i|$ is non-decreasing in $|e_i|$. Therefore, when $i$ does not contribute, she will invite as many neighbors as possible. These discussions finally give us Theorem \ref{corollary:SIM_critical_number}.


\begin{theorem} \label{corollary:SIM_critical_number}
   The Bayesian Nash equilibrium under SIM is: for an arbitrary agent $i$ with ability $q_i$, if the number of invited agents is less than a critical number $t(q_i) = log_{F(q_i)} \frac{c}{\mathcal{M}} + 1$, then $i$ contributes but invites no one; otherwise $i$ invites all her neighbors but does not contribute.
\end{theorem}

\begin{remark}\label{remark:sim}
What Theorem \ref{corollary:SIM_critical_number} indicates is: under SIM, the only thing mattering is the dynamically increasing number of registered agents. (\RNum{1}) During the registration phase, any agent  only needs to monitor the increasing number of registered agents. Before it surpasses her critical number, she should invite no one. 
Once it surpasses her critical number, she should invite every neighbor. 
(\RNum{2}) Then in the competition phase, she contributes only if she did not invited anyone during the registration phase.
\end{remark}

In summary, SIM is very simple for agents to play and it will result in an interesting phenomenon: {\emph{the growth of the registered population scares off some low-ability agents, who then try their best to make the population grow.}}  Thus, SIM can provide a nice incentive whereby low-ability agents try their best to invite, while high-ability agents all contribute their abilities. Eventually, the high-ability agents in a very large crowd stand out, whereby a high-quality solution is obtained.

\section{Collective Invitation Mechanism}
SIM is simple for agents to play, its equilibrium is easy for the principal to predict, and it can also elicit high quality solutions.  However, according to Theorem \ref{corollary:SIM_critical_number} and Remark \ref{remark:sim}, high-ability agents under SIM may still
prefer not to invite anyone, which may limit the highest quality of the solutions. 
In this section, we design another crowdsourcing contest where \emph{all} agents in the crowd prefer to invite their neighbors. We call it collective invitation mechanism (CIM).

\subsection{Contest Mechanism Characterization} We first characterize the workflow of CIM. Like SIM, it also has three phases, but the information exposure and award rule by the mechanism are different.
\begin{enumerate}
\item \textbf{Registration phase}: Agents register in the contest and they can invite their friends to register. The principal collects the invitation record of each registered agent.
\item \textbf{Competition phase}: At the beginning of this phase, the principal constructs and publishes the whole invitation graph $\mathcal{H}$ to all registered agents. Then each agent decides whether to contribute her ability or not.
\item \textbf{Award phase}: Each agent who beats \emph{all the agents not led by herself} is awarded  a prize $\mathcal{M}$. Agents who do not contribute have no reward.
\end{enumerate}

There are two differences of CIM comparing with SIM. One is information exposure. CIM does not publish any information in the registration phase, but publishes $\mathcal{H}$ when the competition phase starts. $\mathcal{H}$ contains not only the number of registered agents but also the structural information among them. The other difference is that CIM may select out multiple high-ability agents and award each a prize $\mathcal{M}$. The allocation rule of the award phase is in Algorithm \ref{alg:CIM}.

Note that $\mathcal{H}$ is constructed by the principal using the invitation record of every agent.  
The principal can collect each agent's invitation record in various ways. For instance, she can ask an agent to report whom she invited and who invited her. 
If two agents' reports are consistent, then there is an edge between them in $\mathcal{H}$.
The principal can also collect the invitation record via certain  technologies such as block chain or digital signature.


\begin{algorithm}[ht]
        \KwIn{invited agents $\mathbf{U}$, invitation graph $\mathcal H$, agents' contribution qualities $q_{i}, \forall i \in \mathbf{U}$}
        \KwOut{a set $\mathbf{W}$ of winners}

        initialize $\mathbf{W} = \emptyset$; \\
        \For{every $i \in \mathbf{U}$}{
            \If{$i$ has contributed}{
                compute $\mathbf{D}_i$ according to $\mathcal{H}$; \\
                \If{$\forall j \in (\mathbf{U-D}_i - \{i\})$, $i$ beats $j$}{
                    $\mathbf{W} \leftarrow \mathbf{W}+\{i\}$; \\
        }}}
        return $\mathbf{W}$.
	\caption{Allocation rule in award phase of CIM}
	\label{alg:CIM}
\end{algorithm}

For ease of notation, denote $\mathbf{P}_i = \mathbf{U} - \mathbf{D}_i - \{i\}$. According to line-5 in Algorithm \ref{alg:CIM}, only agents in $\mathbf{P}_i$ can be competitors of $i$, but the set $\mathbf{P}_i$ is not affected by $i$'s own invitation.
If $i$ beats all $j$ in $\mathbf{P}_i$, she will be selected out as a winner.
Throughout this section, $\mathbf{P}_i$ will be used as a key factor for analyzing agents' behaviors. What is the intuition of designing such an award rule? Recall that by definition, $\mathbf{D}_i$ is the set of agents who can join the contest only when $i$ invites her neighbors. If a mechanism allows a competition between $i$ and agents from $\mathbf{D}_i$, then $i$ will not invite anyone. Therefore, ensuring that any agent $i$ does not compete with agents in $\mathbf{D}_i$ can provide her with incentives to invite. 



According to the algorithm, there are two basic properties of CIM. Firstly, if there is a unique best contributor, then she wins prize $\mathcal{M}$. Secondly, the winners other than the best contributor must be her leaders. This is because for any winner $i$ who is not the best contributor, the best contributor cannot be in $\mathbf{P}_i$, otherwise $i$ cannot be a winner.


We will analyze agents' behaviors under CIM. Different from SIM whose equilibrium is all agents share a common threshold, we will prove that, under CIM, there exists an \textit{asymmetric Bayesian Nash equilibrium} where each agent's best response to the opponents is a threshold strategy and different agents have different thresholds. We will give an explicit form of agents' thresholds.

\subsection{Strategy and Utility}
First we will show that under CIM, the equilibrium strategies are also threshold strategies such that every agent contributes only if her ability surpasses her threshold.

With a fixed invitation graph, consider an agent $i$ with ability $q_i$. Given a strategy of another agent $j$, the probability that $i$ beats $j$ when $i$ contributes is $P(a_j < q_i)$. This value is non-decreasing and continuous in $q_i$.
Under the i.i.d. assumption, the expected utility for $i$ to contribute is
\begin{equation}\label{eq:CIM_contribute}
    \pi_i(q_i) = \mathcal{M} \cdot \prod_{j \in \mathbf{P}_i} P\big(a_j< q_i\big) - c,
\end{equation}
which is also non-decreasing and continuous in $q_i$. Using this utility as well as the continuity and monotonicity of $P$, we can obtain the following proposition.


\begin{proposition}
\label{theorem:CIM contributing strategy}
    Under CIM, the Bayesian Nash equilibrium is a threshold equilibrium, in which each agent contributes only if her ability surpasses her own threshold.
\end{proposition}
\begin{proof}
First consider the case that $i$ contributes with the upper-bound ability $q_i = u$. No matter what $j$ does, it is always that $P(a_j < u)=1$. Thus Eq.(\ref{eq:CIM_contribute}) degenerates to
$$\pi_i(u) = \mathcal{M} - c >0,$$ while $i$'s utility when not contributing is $0$. It means an agent with a close-to-upper-bound ability always contributes.

Next consider a case that agent $i$ contributes with the lower-bound ability $q_i = l$. With this lowest ability, $i$ can beat $j$ only when $j$ does not contribute, which happens with the probability $P(a_j = 0) < 1$. Now Eq.(\ref{eq:CIM_contribute}) becomes $$\pi_i(l) = \mathcal{M} \cdot \prod_{j \in \mathbf{P}_i} P\big(a_j < l\big) - c.$$ 
The sign of this utility is not determined. (1) If $\pi_i(l) \leq 0$, since $\pi_i(q_i)$ is non-decreasing and continuous in $q_i$ and we already know $\pi_i(u) >0$, then there must be a value $r_i\in [l,u)$ such that $\pi_i(r_i) =0$. And $i$ contributes only when her ability $q_i> r_i$. (2) When $\pi_i(l) > 0$, since $q_i\ge l$ and $\pi_i(q_i)$ is non-decreasing and continuous in $q_i$, $i$ will always contribute. We call such agent $i$ \emph{unconditional contributor} under the second condition.
The unconditional contributor is a spacial kind of agent under CIM. Appendix \ref{app_unconditional contributor} discusses more about them.
For the unconditional contributors, we can simply set their thresholds as the lower-bound, i.e, $r_i=l$. This does not change any results.
\end{proof}
This proposition is intuitive.
If an agent $i$ does not contribute, her utility is $0$. If she contributes, her utility $\pi_i(q_i)$ is a non-decreasing and continuous function of her ability $q_i$. Specifically, if she has the highest ability $q_i=u$, then $\pi_i$ is positive. The non-decreasing monotonicity of $\pi_i$ ensures that there is a threshold value $r_i\ge l$ such that $\pi_i(q_i) \ge 0$ only if $q_i \ge r_i$. Proposition \ref{theorem:CIM contributing strategy} says that we only need to focus on the threshold strategies. Note that the thresholds for different agents are not necessarily identical under CIM.

With this proposition and Definition \ref{df:beat} we can further know that, the probability for $i$ to beat $j$ by contributing is $ P\left(q_j \le \max \{q_i, r_j\}\right) = F_{q_j}(\max \{q_i, r_j\}).$
Therefore under the i.i.d. assumption, Eq.\eqref{eq:CIM_contribute} can be reformulated into
\begin{equation}     \label{eq:CIM_contribute_utility}
    \pi_i(q_i) = \mathcal{M} \cdot \prod_{j \in \mathbf{P}_i} F_{q_j}\big(\max \{q_i, r_j\}\big) - c.
    \end{equation}
This formula is an agent's expected utility regarding others' thresholds. In the next subsection, it will be used as a key component for analysing agents' invitation incentives.

\subsection{Patterns of Asymmetric Threshold}
Proposition \ref{theorem:CIM contributing strategy} is of great importance. It restricts the strategy space to the threshold form. Now based on structural analysis of leading relations among agents, we describe several basic patterns about the asymmetric thresholds of all agents.

The following lemma leverages the structural position of any agent, and it
shows that threshold $r_i$ of any agent $i$ has an upper-bound.
\begin{lemma}\label{app_lm:threshold upper bound}
For any agent $i$ under CIM, her ability threshold $r_i \leq F^{-1}\left(\sqrt[\leftroot{0}\uproot{4}|\mathbf{P}_i|]{\frac{c}{\mathcal{M}}}\right)$, where $\mathbf{P}_i = \mathbf{U} - \mathbf{D}_i - \{i\}$.
\end{lemma}

\begin{proof}
    Prove by contradiction. Let $s_i = F^{-1}(\sqrt[\leftroot{0}\uproot{4}|\mathbf{P}_i|]{\frac{c}{\mathcal{M}}})$. Suppose that there exists an agent $i$ whose $r_i > s_i$. When $s_i < q_i < r_i$, the expected utility for $i$ to contribute is
    {\small \begin{align}
        \pi_i(q_i) & =  \mathcal{M} \cdot \prod_{j \in \mathbf{P}_i} F_{q_j}\big(\max\{r_j, q_i\}\big) - c \notag \\
       & \geq \mathcal{M} \cdot \prod_{j \in \mathbf{P}_i} F_{q_j}(q_i) - c \notag 
        > \mathcal{M} \cdot \prod_{j \in \mathbf{P}_i} F_{q_j}(s_i) - c \notag \\
       & =  \mathcal{M} \cdot \prod_{j \in \mathbf{P}_i}\sqrt[\leftroot{0}\uproot{8}|\mathbf{P}_i|]{\frac{c}{\mathcal{M}}} - c = 0. \notag
    \end{align}}Thus $i$ will contribute with $q_i < r_i$. This contradicts that $r_i$ is the threshold.
\end{proof}
Here the only variable is $|\mathbf{P}_i|$, which is the number of $i$'s competitors. These agents  are not affected by $i$'s invitation. This means the upper-bound of an agent's threshold is independent from her invitation decision.
Moreover, if two agents have leading relation, then the following lemma can be used to compare their thresholds.
\begin{lemma} \label{app_ps:threshold relation in dominating relation}
    In the Bayesian Nash equilibrium of CIM, if agent $i$ leads agent $j$, then $j$ has higher threshold than $i$.
\end{lemma}
\begin{proof}
    Suppose two agents $i$ and $j$ where $i$ leads $j$, then $j \in \mathbf{D}_i$, $\mathbf{D}_j \subset \mathbf{D}_i$ by definition, and thus $\mathbf{P}_i \subset \mathbf{P}_j$. First consider $j$. Since $r_j$ is the threshold of $j$, if the $j$'s ability $q_j=r_j$, then her utility of contributing is
    {\small \begin{align} 
       \pi_j(q_j=r_j) = \mathcal{M} \cdot \prod_{k \in \mathbf{P}_j} F_{q_k}\big(\max\{r_k,r_j\}\big) - c\geq 0.\notag
    \end{align}}The product in this equality can be factorized as
    {\small \begin{align} \label{app_eq:product in above}
        \prod_{k_1 \in \mathbf{P}_i}F_{k_1}\big(\max\{r_{k_1},r_j\}\big) \cdot \prod_{k_2 \in \mathbf{P}_j,  k_2 \notin \mathbf{P}_i} F_{k_2}\big(\max\{r_{k_2},r_j\}\big).
    \end{align}}Now consider $i$. If $i$'s ability $q_i = r_j$, then her utility of contributing is
   {\small \begin{align}
        \pi_i(q_i=r_j) = \mathcal{M} \cdot \prod_{k_1 \in \mathbf{P}_i}F_{k_1}\big(\max \{r_{k_1},r_j\}\big) - c,\notag
   \end{align}}where the product is exactly the same as the first term in Eq.\eqref{app_eq:product in above}. Since the second term in Eq.\eqref{app_eq:product in above} is a probability less than $1$, we have $\pi_i(q_i=r_j) > \pi_j(q_j=r_j) \geq 0$. The utility of $i$ can only be $0$ if she does not contribute. This means if $q_i = r_j$, $i$ will contribute. Therefore, $r_i < r_j$.
 \end{proof}

This lemma shows that, we can rank different agents' thresholds via sorting them according to the leading relations. This further reflects the fact that the structural information can help us to calculate the equilibrium behaviors.  Specifically, agents not leading anyone must be leaf nodes in the ranking tree. The following definition formally defines the leaf agents.
\begin{definition} (leaf agent)
    \label{df:leaf type}
    An agent $i$ is an leaf agent if she leads  no one (i.e. $\mathbf{D}_i = \emptyset$).
\end{definition}
Note that a leaf  agent in the ranking tree does not necessarily mean she has no friends. It only means this agent is not a cut node in the invitation graph. For an agent of the leaf type, unilaterally altering her invitation does not change the group of invited agents. Proposition \ref{app_lm:leaf type threshold} proves that all leaf agents $z$ have an identical threshold, and the proof details can be found in Appendix \ref{appendix: leaf_node_threshould_proof}.
\begin{proposition}\label{app_lm:leaf type threshold}
In the Bayesian Nash equilibrium under CIM, every agent $z$ who is a leaf node in the ranking tree of the invitation graph has an identical threshold $r_{z} = F^{-1}\left(\sqrt[\leftroot{-1}\uproot{3}|\mathbf{U}|-1]{\frac{c}{\mathcal{M}}}\right)$, which is the highest among all agents' thresholds. 
\end{proposition}

We know that for all leaf agents, the value $|\mathbf{P}_i|=|\mathbf{U}|-1$, which is the largest among all agents. Combining with Lemma \ref{app_lm:threshold upper bound}, we can know a leaf agent's threshold is larger than any non-leaf agent's threshold.
Moreover, the above proposition holds for any $\mathcal{H}$, including an extreme case where every agent is of leaf type (i.e., there is no cut node in $\mathcal{H}$) and all agents share a common threshold.

\begin{remark}\label{remar:CIM}
In the Bayesian Nash equilibrium under CIM, (\RNum 1) An agent contributes only if her ability surpasses her threshold. (\RNum 2) As long as some agents are not of leaf type, there is no common threshold.
(\RNum 3) Leaf agents have the highest threshold.
\end{remark}

\subsection{Invitation in Asymmetric Equilibrium}
The  three points in Remark \ref{remar:CIM} are straightforward from the previous results. Point (\RNum 1) is from Proposition \ref{theorem:CIM contributing strategy}, point (\RNum 2) is from Lemma \ref{app_ps:threshold relation in dominating relation}, and point (\RNum 3) is from Proposition \ref{app_lm:leaf type threshold}. Especially, from point (\RNum 3), it is apparent that under CIM, no agent wants to be of leaf type, which incentivizes each agent to invite her neighbors so as to make herself leading some agents.
The following Lemma formally proves this.


\begin{lemma}\label{lemma:invite decrease threshold}
Under CIM, for any agent, she can decrease her threshold by inviting her neighbors.
\end{lemma}
\begin{proof}
Consider an arbitrary agent $i$. Given the other agents' invitations, when $i$ invites no one, she becomes leaf type. Suppose in this case the whole agent set is $\mathbf{U}$ and $i$'s threshold is $r_i$. When $i$ invites some (or all) of her neighbors, suppose the whole agent set is $\mathbf{U}'$ and her threshold is $r_i'$. It is straightforward that $\mathbf{U} \subseteq \mathbf{U}'$.

When $i$ invites no one, $i$ is a leaf agent. By Proposition \ref{app_lm:leaf type threshold}, we have $r_i = F^{-1}(\sqrt[\leftroot{-3}\uproot{3}|\mathbf{U}|-1]{\frac{c}{\mathcal{M}}})$ and $F(r_i)^{|\mathbf{U}|-1} = \frac{c}{\mathcal{M}}$. Note that $r_i$ is also the upper-bound threshold shown in Lemma \ref{app_lm:threshold upper bound}.

When $i$ invites her neighbors, if $\mathbf{U} = \mathbf{U'}$, $i$ is still of leaf type and $r'_i = r_i$. Otherwise, $\mathbf{U} \subset \mathbf{U'}$, but $i$'s competitor set is still $\mathbf{P}_i$ and some agents in $\mathbf{P}_i$ would still be of leaf type. According to Proposition \ref{app_lm:leaf type threshold}, a leaf agent $z \in \mathbf{P}_i$ has the threshold $r'_z = F^{-1}(\sqrt[\leftroot{-3}\uproot{3}|\mathbf{U'}|-1]{\frac{c}{\mathcal{M}}})$. Since $|\mathbf{U}| < |\mathbf{U'}|$, it is easy to know $r'_z  > r_i$.
Based on this, if $i$'s ability $q_i = r_i$, then her utility for contributing and inviting is
{\small \begin{equation}
\begin{aligned}
 & \pi_i (\mathbf{U}', q_i=r_i) \\
 = &\mathcal{M} \cdot F_{q_z}\left(\max\{r'_z, r_i\}\right) \cdot \prod_{
 k \in \mathbf{P}_i, k \neq z } F_{q_k}\left(\max\{r'_k, r_i\}\right) - c \notag\\
  \geq& \mathcal{M} \cdot F_{q_z}(r'_z) \cdot  \prod_{
 k \in \mathbf{P}_i, k \neq z } F_{q_k}(r_i) - c \\
> &\mathcal{M} \cdot \prod_{k \in \mathbf{P}_i} F_{q_k}(r_i) - c =0. \notag
\end{aligned}
\end{equation}}In the above formula, line-2
classifies competitors into two groups, i.e., leaf agent $z$ and other agents.
Line-3 to line-4 is by the previous result that $r'_z  > r_i$. The equality part in line-4 is by the facts that $|\mathbf{P}_i|=|\mathbf{U}|-1$ and $r_i = F^{-1}(\sqrt[\leftroot{-3}\uproot{3}|\mathbf{U}|-1]{\frac{c}{\mathcal{M}}})$ in Proposition \ref{app_lm:leaf type threshold}.

Now we know $\pi_i(\mathbf{U}', q_i=r_i)>0$. Since $\pi_i$ is monotonically increasing in $q_i$, there must be a smaller value $r'_i < r_i$ that makes $\pi_i(\mathbf{U}', q_i=r'_i)=0$. This $r'_i$ is the threshold when $i$ invites her neighbors. Therefore, an agent can decrease her threshold by inviting her neighbors.
\end{proof}

This lemma points out the fact that inviting neighbors reduces the requirement on an agent's ability.
The intuition is: when an agent invites her neighbors, there is a chance for her to lead some agents and avoid being a leaf agent who is proved to have the upper-bound threshold. The following theorem goes one step further and formally shows that in equilibrium, all agents will invite someone.
\begin{theorem}[Collective Invitation] \label{th:CIM invitation incentive}
Under CIM, every agent invites some or all of her neighbors in Bayesian Nash equilibrium.
\end{theorem}
\begin{proof}
By Lemma \ref{lemma:invite decrease threshold}, agent $i$'s threshold when she invites is less equal  than that when she doesn't invite, i.e., $r'_i \leq r_i $. Given  $q_i$, there are three cases of the order of $q_i$, $r_i$ and $r_i'$.

1) $q_i \leq r'_i \leq r_i$. In this case, $i$ doesn't contribute either she makes invitation or not. Then $\pi_i(\mathbf{U}', q_i) = \pi_i(\mathbf{U}, q_i)=0$, where $\mathbf{U}'$ is the agent set when $i$ invites, and $\mathbf{U}$ is that when $i$ does not invite.

2) $r'_i < q_i \leq r_i$. In this case, if she doesn't make any invitation in the registration phase, $i$ won't contribute in the competition phase, and then her utility is $0$; if she invites in the registration phase, $i$ will contribute in the competition phase and then obtain utility greater than $0$. Therefore, $\pi_i(\mathbf{U}', q_i) > \pi_i(\mathbf{U}, q_i)$.

3) $r'_i \leq r_i < q_i$. In this case, $i$ will contribute in the competition phase no matter whether she has invited or not. $r_i$ is $i$'s threshold when she is a leaf agent, thus by
the threshold upper-bound in Lemma \ref{app_lm:threshold upper bound}, for any $k \in \mathbf{P}_i$, $r_k \leq r_i$. Thus, $r_k < q_i$, and
{\small \begin{align}
\pi_i\big(\mathbf{U}', q_i\big) = & \mathcal{M} \cdot \prod_{k \in \mathbf{P}_i} F_{q_k}\big(\max\{r'_k, q_i\}\big) - c \notag \\
\geq & \mathcal{M} \cdot \prod_{k \in \mathbf{P}_i} F_{q_k}(q_i) - c\notag \\
= & \mathcal{M} \cdot \prod_{k \in \mathbf{P}_i} F_{q_k}\big(\max\{r_k, q_i\}\big) - c = \pi_i\big(\mathbf{U}, q_i\big).\notag
\end{align}}Considering all the these cases, inviting friends is better off for every agent under CIM.
\end{proof}

From Theorem \ref{th:CIM invitation incentive}, we know that CIM provides a good incentive such that in the Bayesian Nash equilibrium, every agent would invite some, if not all, of her neighbors.

\subsection{Contribution in Asymmetric Equilibrium}
Now we derive an expression of an arbitrary agent's threshold.
For an arbitrary $i$, if her ability just meets her threshold, i.e., $q_i=r_i$, her utility is $0$. Then Eq.(\ref{eq:CIM_contribute_utility}) can be rewritten as
{\small \begin{equation} \label{eq:solving threshold for non-leaf agent}
    \begin{aligned}
        \pi_i(r_i)
         = \mathcal{M} \cdot
         \prod_{j \in \mathbf{J}_i }
            F_{q_j}(r_{j}) \cdot
            \prod_{j \in \mathbf{Q}_i}
            F_{q_j}(r_i) - c = 0,
    \end{aligned}
\end{equation}}where $\mathbf{J}_i$ is the set of $i$'s competitors whose thresholds are larger than $r_i$, and $\mathbf{Q}_i$ is the set of $i$'s competitors whose thresholds are not larger than $r_i$. That is, $\mathbf{P}_i=\mathbf{J}_i\cup\mathbf{Q}_i$. Under the i.i.d. assumption, $F$ is identical for different $j$. Then we can rewrite Eq.(\ref{eq:solving threshold for non-leaf agent}) as
{\small $$
F(r_i)^{|\mathbf Q_i|}={\frac{c}{\mathcal{M}} \Bigm/ \prod_{j \in \mathbf{J}_i} F(r_{j})}.
$$}This formula together with Proposition \ref{app_lm:leaf type threshold} immediately leads to the following theorem, which is our major result.
\begin{theorem}[Private Threshold] \label{theorem:threshold for non-leaf agent}
    In the Bayesian Nash equilibrium of CIM, the ability threshold of any leaf agent $z$ is
    {\small \begin{equation}\label{eq:leaf node threshold}r_{z} = F^{-1}\left(\sqrt[\leftroot{-3}\uproot{15}|\mathbf{U}|-1]{\frac{c}{\mathcal{M}}}\right),
    \end{equation}}\\
    while that of any non-leaf agent $i$ is
    {\small \begin{equation} \label{eq:calculate threshold for non-leaf agent}
    r_i=F^{-1}\left(\sqrt[\leftroot{-3}\uproot{21}|\mathbf Q_i|]{\frac{c}{\mathcal{M}} \Bigm/ \prod_{j \in \mathbf{J}_i } F(r_{j})}\right).
    \end{equation}}
\end{theorem}
The special case where $|\mathbf{Q}_i|=0$ is handled in Appendix \ref{app_unconditional contributor}.
Recall that $\mathbf{J}_i$ is the set of competitors whose thresholds are greater than $i$'s.
Theorem \ref{theorem:threshold for non-leaf agent} tells us that, to decide $r_i$,
we  need to know $r_j$ of each $j\in \mathbf{J}_i$, as well as the size of $\mathbf{Q}_i$. For this purpose,
we can devise an algorithm to traverse all nodes and recursively compute each $r_i$ by leveraging Theorem \ref{theorem:threshold for non-leaf agent}.


So far, we have shown how agents make invitation decisions in the registration phase and also given the formula of their ability thresholds. But still, we have not seen the detailed invitation behaviors and contributions from agents in equilibrium. In the next section, we will propose a computational method to characterize the agents' equilibrium behaviors under CIM.

\section{Equilibrium Computation}
Now we propose a computational method for automatically searching the asymmetric equilibrium of CIM.
We first show that it is possible to rank agents' ability thresholds even without knowing the values of these thresholds. The key idea is to find another computable numerical value which is monotone in $r_i$. Since this mediator value has a one-to-one correspondence relation with the agent's  threshold, we can rank the thresholds by ranking agents' mediator values. We then design an algorithm to compute the agents' thresholds in equilibrium by using this ranking and the ability threshold's formula in Eq.(\ref{eq:calculate threshold for non-leaf agent}). These thresholds finally determines the equilibrium and the performance of CIM.


\subsection{Monotonicity of Thresholds}\label{subsection:monotone}
In the competition phase of CIM, for agent $i$, if all the agents led by her (i.e. $\mathbf{D}_i$) do not have enough ability to contribute, we say $i$ is a \emph{leader of mediocre} (\emph{lom}).
Denote the probability for $i$ to be a \emph{lom} by $p_i^{lom}$, which is calculated as
\begin{equation}\label{eq:p_subtree_waive}
    p_i^{lom}=
    P\left(\forall k\in \mathbf{D}_i, q_k<r_k\right)
    =\prod_{k \in \mathbf{D}_i} F_{q_k}(r_k).
\end{equation}
Recall that $F_{q_k}(r_k)$ is the probability that $k$'s ability is below her threshold $r_k$, thus the product is the probability that no agent in $\mathbf{D}_i$ can contribute. 

The intuition behind the probability $p_i^{lom}$ is: if agent $i$ has a small $p_i^{lom}$, it is more likely that $i$ invites some contributors who have enough ability. An agent with small $p_i^{lom}$ has a more important social position, for the subtree she leads is more competitive. Therefore, the invitation from an agent with small $p_i^{lom}$ is more likely to scare off her competitors (i.e. the thresholds of agents in $\mathbf{P}_i$ are high due to $i$'s invitation).
When $i$'s competitors are scared off,
$i$'s threshold becomes smaller.
The following two lemmas give a formal proof for that agents with small $p_i^{lom}$ has low thresholds. 


\begin{lemma}\label{app_lm:non dominating threshold relation}
    For any two agents $i$ and $j$ who don't have leading relation (i.e. $i \nsucc j  \land j \nsucc i$), if $p_i^{lom} < p_j^{lom}$, then $r_i < r_j$.
\end{lemma}

\begin{lemma} \label{app_lm:identical threshold}
   If $
         p_i^{lom} = p_j^{lom}
    $, then $r_i = r_j$.
\end{lemma}

These  two lemmas are simple but profound, they are crucial for our major result. Recall that our final objective is to compute agents' thresholds in the asymmetric BNE, where each agent has her own threshold.  $p_i^{lom}$ is a breakthrough point. This is because, given  a ranking tree, different agents' $p_i^{lom}$ can be computed in a bottom-up recursive way. More importantly, these two lemmas  reveal that, to sort agents' thresholds in equilibrium, we only need to rank their $p_i^{lom}$, which is more operational. This makes it possible to design a recursive algorithm for equilibrium computation. The proofs of these lemmas use similar techniques, but they need some calculation and are space consuming. We left the proofs into Appendix \ref{appendix:plomless} and \ref{appendix:plomeq}.

From Lemma \ref{app_ps:threshold relation in dominating relation}, Lemma \ref{app_lm:non dominating threshold relation} and Lemma \ref{app_lm:identical threshold}, it's easy to derive that, in any equilibrium under CIM, $\forall i, j \in \mathbf{U} (i \neq j)$,
\begin{equation} \label{app_eq:threshold monotonicity}
    \begin{aligned}
        r_i = r_j & \Leftrightarrow p_i^{lom} = p_j^{lom}, \\
        r_i < r_j & \Leftrightarrow p_i^{lom} < p_j^{lom}.
    \end{aligned}
\end{equation}
Thus we can directly have the following Proposition \ref{proposition:global threshold monotonicity} which serves as a corner stone for designing an algorithm to compute the BNE of CIM.

\begin{proposition}[Monotonicity of Threshold] \label{proposition:global threshold monotonicity}
    In the Bayesian Nash equilibrium of CIM, for any two agents $i$ and $j$, $r_i<r_j$ iff $p_i^{lom}<p_j^{lom}$.
\end{proposition}
The significance of Proposition \ref{proposition:global threshold monotonicity} is, it shows that if we want to rank all agents' thresholds, we just need to rank the values of all $p_i^{lom}$, which is more easily accessible than $r_i$.

\subsection{Recursive Algorithm}
Based on Theorem \ref{theorem:threshold for non-leaf agent} and the fact that the leaf agents have the highest threshold, we design Algorithm \ref{app_alg:solving CIM} to compute the Bayesian Nash equilibrium of CIM.

The algorithm has two key temp sets $\mathbf A$ and $\mathbf T$, which are updated iteratively. Set $\mathbf A$ stores all agents whose thresholds are already known and set $\mathbf T$ contains agents whose thresholds will be computed in the current iteration. In each iteration, the current $\mathbf{T}$ is computed with information from the previous $\mathbf{A}$. Then the obtained current $\mathbf T$ is added to $\mathbf A$.

The algorithm has four components. In initialization,
the ranking tree $\mathcal T$ is generated. 
Set $\mathbf A$ is initialized with all leaf agents (i.e. leaf nodes in $\mathcal T$), whose common threshold is computed by Eq.(\ref{eq:leaf node threshold}).
In each iteration, using the thresholds of agents in $\mathbf A$, the second component 
checks whose $p_i^{lom}$ values are ready to compute (line \changreflcolor{black}\ref{app_code lin Algirhtm CIM lom: if}\changreflcolor{blue}), and computes them by Eq.(\ref{eq:p_subtree_waive}) (line \changreflcolor{black}\ref{app_code lin Algirhtm CIM lom: prob}\changreflcolor{blue}). After that, in the third component 
, the obtained $p_i^{lom}$ values are sorted (line \changreflcolor{black}\ref{app_code lin Algirhtm CIM rank: start}\changreflcolor{blue}), and agents with the largest $p_i^{lom}$  are selected out and moved to the temp set $\mathbf T$ (line \changreflcolor{black}\ref{app_code lin Algirhtm CIM rank: end}\changreflcolor{blue}). Finally, still using thresholds of agents in $\mathbf A$ and by  Eq.(\ref{eq:calculate threshold for non-leaf agent}), the fourth component 
computes the thresholds of all nodes in  $\mathbf T$, and then attaches $\mathbf T$ to  $\mathbf A$. The following work flow explains this algorithm in detail.



\begin{enumerate}[leftmargin=*]
    \item {\textbf{Initialization}} (lines \changreflcolor{black}\ref{app_code lin Algirhtm CIM initial: start} - \ref{app_code lin Algirhtm CIM initial: end}\changreflcolor{blue}): With the set of all invited agents $\mathbf U$ and the invitation graph $\mathcal H$, it is easy to compute the leaf agent set $\mathbf Z$, and $\mathbf{D}_i$ and $\mathbf{P}_i$ for every $i$. It is also shown that the threshold of leaf agents $\mathbf Z$ can be directly computed by Eq.(\ref{eq:leaf node threshold}), which is done in the loop.

    \item {\textbf{Computing $p_i^{lom}$}} (lines \changreflcolor{black}\ref{app_code lin Algirhtm CIM lom: start} - \ref{app_code lin Algirhtm CIM lom: end}\changreflcolor{blue}): If the thresholds $r_k$ of $i$'s descendants $k\in\mathbf D_i$ are all known,
    then using Eq. (\ref{eq:p_subtree_waive}), we can calculate  $p_i^{lom}$ for agent $i$.   Lines \changreflcolor{black}\ref{app_code lin Algirhtm CIM lom: start} - \ref{app_code lin Algirhtm CIM lom: if}\changreflcolor{blue} traverse $\mathbf{U}$ and find every $i$ whose threshold is not yet obtained (i.e., $r_i=nil$) but can be computed  (i.e., $\forall k\in \mathbf D_i, r_k \neq nil$). Line \changreflcolor{black}\ref{app_code lin Algirhtm CIM lom: prob}\changreflcolor{blue} computes $p_i^{lom}$ of these agents.
    \label{app_component for lom of CIM Equilibrium algorithm}

    \item {\textbf{Sorting $p_i^{lom}$}}  (lines \changreflcolor{black}\ref{app_code lin Algirhtm CIM rank: start} - \ref{app_code lin Algirhtm CIM rank: end}\changreflcolor{blue}): With Proposition \ref{proposition:global threshold monotonicity}, we can sort  $r_i$ via sorting   $p_i^{lom}$. Given all the agents whose $r_i$ is not known but $p_i^{lom}$ have been obtained,
    this component ranks them and selects out those with the highest $p_i^{lom}$. 
    \label{app_component for ranking agents of CIM Equilibrium algorithm}

    \item {\textbf{Computing $r_i$}} (lines \changreflcolor{black}\ref{app_code lin Algirhtm CIM assign threshold: start} - \ref{app_code lin Algirhtm CIM assign threshold: end}\changreflcolor{blue}): With the largest $p_i^{lom}$ agents, this component computes $r_i$ for these agents following Eq.(\ref{eq:calculate threshold for non-leaf agent}). In each iteration of the \textbf{while} loop, there is a new set of the largest $p_i^{lom}$ agents $\mathbf T$. Set $\mathbf A$ is updated in each iteration, and gradually expands until all agents' thresholds have been obtained.
\end{enumerate}

\begin{algorithm}[ht]
        \KwIn{invited agents $\mathbf{U}$, invitation graph $\mathcal H$}
        \KwOut{thresholds $\vec{r} = (r_1, \cdots, r_i, \cdots, r_{|\mathbf{U}|})$}
        generate the ranking tree $\mathcal T$; \label{app_code lin Algirhtm CIM initial: start}\\
        compute $\mathbf{D}_{i}$ and $\mathbf{P}_{i}$ for all $i \in \mathbf{U}$; \\
        initialize flag $\phi_i=0$ and $r_i = nil$ for all $i \in \mathbf{U}$ ; \\
        find leaf node set $\mathbf{Z}$;\\
         \For{each leaf node $i\in \mathbf Z$}
         {assign $r_i$ with the value in Eq.(\ref{eq:leaf node threshold});\\       
        }
        initialize threshold-known node set $\mathbf A=\mathbf Z$;\label{app_code lin Algirhtm CIM initial: end}

        \While{$\mathbf{A}\neq\mathbf{U}$}{
            \For{each $i$ whose $\phi_i=0$ and $r_i = nil$\label{app_code lin Algirhtm CIM lom: start}}{
                \If{$\forall k\in \mathbf D_i, r_k \neq nil$\label{app_code lin Algirhtm CIM lom: if}}{
                    compute $i$'s  $p_i^{lom}$ value by Eq.(\ref{eq:p_subtree_waive});\label{app_code lin Algirhtm CIM lom: prob}\\

                    set $i$'s flag to $\phi_i=1$;\label{app_code lin Algirhtm CIM lom: end}\label{app_code lin Algirhtm CIM lom: set flag to 1}\\

                }
            }

            rank all $i$ whose $\phi_i=1$ by descending $p_i^{lom}$;\label{app_code lin Algirhtm CIM rank: start}\\
            move all top $i$ into temp set $\mathbf T$;\label{app_code lin Algirhtm CIM rank: end}\\

            \For{all $i\in \mathbf T$\label{app_code lin Algirhtm CIM assign threshold: start}}{use $\mathbf{J}_i=\mathbf{A} \cap \mathbf{P}_{i}$ to compute $r_i$ by Eq.(\ref{eq:calculate threshold for non-leaf agent});\label{app_code lin Algirhtm CIM assign threshold: threshold}\\
            reset $\phi_i=0$;\label{app_code lin Algirhtm CIM assign threshold: reset flag}\\
            }
            update $\mathbf{A}=\mathbf{A}\cup\mathbf T$;
            then reset $\mathbf T=\emptyset$; \label{app_code lin Algirhtm CIM assign threshold: end}\\
        }
        \textbf{return} $\vec{r}$.
	\caption{Computing Threshold in BNE of CIM }
	\label{app_alg:solving CIM}
\end{algorithm}


We would like to emphasize to the readers that, in each iteration of the \textbf{while} loop, we only select out the highest $p_i^{lom}$ agents to form $\mathbf T$ and compute their thresholds $r_i$. Recall that $r_i$ is monotone in $p_i^{lom}$ (subsection \ref{subsection:monotone}). Therefore, the thresholds computed in any previous iterations must be higher than that computed in the current iteration. 
Technically, in each iteration of the \textbf{while} loop, every $r_i$ is computed using the previously obtained opponent set ${\bf J}_i$, in which all agents must have higher thresholds than $i$. Such a computation is an implementation our major theoretical result Theorem \ref{theorem:threshold for non-leaf agent}. It is also the spirit for the recursion in Algorithm \ref{app_alg:solving CIM}. 
Given the above analysis about the \textbf{while} loop and the inner loops, it is not hard to know the time complexity of Algorithm \ref{app_alg:solving CIM} is $\mathcal{O}(n^3)$ where $n$ represents the crowd size. The detailed analysis is in Appendix \ref{Complexity analysis}.

With Proposition \ref{proposition:global threshold monotonicity}, Theorem \ref{theorem:threshold for non-leaf agent} and the characterization of Algorithm \ref{app_alg:solving CIM}, now we summarize the  result for the computation of BNE under CIM.
\begin{theorem}
\label{th:CIM equilibrium}
Under CIM, it is a Bayesian Nash equilibrium that every agent invites some or all of her neighbors and their ability thresholds are computed by Algorithm \ref{app_alg:solving CIM}.
\end{theorem}

\subsection{Efficient  Computation on Large Graphs}
In this subsection, we first analyze to what extent agents are willing to invite  friends and whether it is possible that ``all agents invite all neighbors'' forms the Bayesian Nash equilibrium. We then adapt Algorithm \ref{app_alg:solving CIM} to large graphs.

Recall in Eq.\eqref{eq:CIM_contribute_utility}, the expected utility for $i$ to contribute with ability $q_i$ depends on the value of equilibrium thresholds of $i$'s competitors $j \in \mathbf{P}_i$. 
By definition of Bayesian Nash equilibrium,
we  need to verify that \emph{inviting all is the best response when all other agents invite all}. This means we need to examine whether there exists a profitable unilateral deviation from ``fully inviting'' to ``partially inviting''. One needs to traverse every agent's whole invitation action space (i.e. the power set of the neighbor set), which is exponential to the number of neighbors. Thus for large graphs, it is crucial to simplify this procedure.

According to subsection \ref{subsection:ranking tree}, if different invitation graphs have identical ranking tree structure, then these different graphs have identical input for Algorithm \ref{app_alg:solving CIM}.
This gives us the following proposition.

\begin{proposition} \label{app_ob:graph identical thresholds}
If two different invitation graphs $\mathcal{H}$ and $\mathcal{H}'$ can generate an identical ranking tree $\mathcal{T}$, and agent $i$ in $\mathcal{H}$ and agent $i'$ in $\mathcal{H}'$ correspond to a same node in $\mathcal{T}$, then these two agents $i$ and $i'$ have the same ability threshold in equilibrium.
\end{proposition}



We give the following definition about the nodes in the ranking tree, which can effectively simplify the verification.
Denote the sets of direct descendants  of $i$ in $\mathcal{T}$ by $\mathbf{D}^{1}_i$.


\begin{definition} (Agent Type) \label{app_df:type of agents}
    Two agents $i$ and $j$ are of the same type if the sub-trees rooted from them are isomorphic.
\end{definition}
Specifically, $i$ and $j$ are of the same type if (\romannumeral 1) they are both leaf nodes, or (\romannumeral 2) there is a bijection $\mathcal{R}_{i,j}: \mathbf{D}^{1}_i \to \mathbf{D}^{1}_j$ such that for every pair of nodes $({d}^{1}_i, {d}^{1}_j) \in \mathcal{R}_{i,j}$ where node ${d}^{1}_i \in \mathbf{D}^{1}_i$ and ${d}^{1}_j \in \mathbf{D}^{1}_j$, ${d}^{1}_i$ and ${d}^{1}_j$ are of the same type. Thus we can  verify if two nodes are of the same type by a bottom-up recursion on $\mathcal{T}$. See the following example.
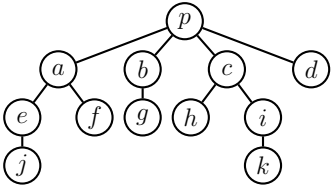
\begin{figure}[h]
    \centering
        \begin{tikzpicture}[thick, global scale=0.8]
    	    \node (p) at (0.0,0.0)[circle,draw,minimum size=6mm]{};
    		\node (a) at (-2.1,-0.8)[circle,draw,minimum size=6mm]{}
    		edge[-] (p);
    		\node (b) at (-0.7,-0.8)[circle,draw,minimum size=6mm]{}
    		edge[-] (p);
    		\node (c) at (0.7,-0.8)[circle,draw,minimum size=6mm]{}
    		edge[-] (p);
    		\node (d) at (2.1,-0.8)[circle,draw,minimum size=6mm]{}
    		edge[-] (p);
    		\node (e) at (-2.7,-1.6)[circle,draw,minimum size=6mm]{}
    		edge[-] (a);
    		\node (f) at (-1.5,-1.6)[circle,draw,minimum size=6mm]{}
    		edge[-] (a);
    		\node (g) at (-0.7, -1.6)[circle,draw,minimum size=6mm]{}
    		edge[-] (b);
    		\node (h) at (0.1, -1.6)[circle,draw,minimum size=6mm]{}
    		edge[-] (c);
    		\node (i) at (1.3, -1.6)[circle,draw,minimum size=6mm]{}
    		edge[-] (c);
    		\node (j) at (-2.7, -2.4)[circle,draw,minimum size=6mm]{}
    		edge[-] (e);
    		\node (k) at (1.3, -2.4)[circle,draw,minimum size=6mm]{}
    		edge[-] (i);
    		
    		\node (1) at (0.0,0.0){\large $p$};
    		\node (2) at (-2.1,-0.8){\large $a$};
    		\node (3) at (-0.7,-0.8){\large $b$};
    		\node (4) at (0.7,-0.8){\large $c$};
    		\node (5) at (2.1,-0.8){\large $d$};
    		\node (6) at (-2.7,-1.6){\large $e$};
    		\node (7) at (-1.5,-1.6){\large $f$};
    		\node (8) at (-0.7, -1.6){\large $g$};
    		\node (9) at (0.1, -1.6){\large  $h$};
    		\node (10) at (1.3, -1.6){\large  $i$};
    		\node (12) at (-2.7, -2.4){\large  $j$};
    		\node (13) at (1.3, -2.4){\large $k$};
        \end{tikzpicture}
    \caption{An example ranking tree $\mathcal{T}$ for Definition \ref{app_df:type of agents}.}
    \label{fig:definition type illustration}
\end{figure}

\begin{example} 
(\RNum{1}) Agents $j$, $f$, $g$, $h$, $k$ and $d$ are all leaf type. \\(\RNum{2}) For $e$, $b$ and $i$, the sub-trees rooted from them are isomorphic, so they are of the same type. Specifically, $\mathbf{D}^{1}_e = \{j\}$, $\mathbf{D}^{1}_b = \{g\}$, $\mathbf{D}^{1}_i = \{k\}$. There is a bijection $\mathcal{R}_{e,b} =\{(j,g)\}$, where $j$ and $g$ are of the same type, so $e$ and $b$ are of the same type; similar for $e$ and $i$, as well as $b$ and $i$. \\(\RNum{3}) For $a$ and $c$, the sub-trees rooted from them are isomorphic, so they are of the same type too.
    Specifically, $\mathbf{D}^{1}_a = \{e,f\}$, $\mathbf{D}^{1}_c = \{h,i\}$. Recursively from the known fact that $e$ and $i$ are of the same type, and  $f$ and $h$ are of the same type, there is a bijection $\mathcal{R}_{a,c} = \{(e,i),(f,h)\}$.
\end{example}

With Definition \ref{app_df:type of agents}, all agents in a  ranking tree can be categorized into different classes, each capturing a type. More importantly, the type-checking work can be done efficiently with a bottom-up recursion algorithm.

Generally, for two agents $i$ and $j$ of the same type, if there is that $r_{{d}^{1}_i} = r_{{d}^{1}_j}$ for every pair of $({d}^{1}_i, {d}^{1}_j) \in \mathcal{R}_{i,j}$, then by Proposition \ref{proposition:global threshold monotonicity}, we know $p_{d_i^{1}}^{lom} =p_{d_j^{1}}^{lom}$. Moreover, according to the definition of $p_i^{lom}$ in Eq.\eqref{eq:p_subtree_waive}, we have
{\small $$
    \prod_{h \in \mathbf{D}_{{d}^{1}_i}} F_{q_h}(r_h) = \prod_{k \in \mathbf{D}_{{d}^{1}_j}} F_{q_k}(r_k). \notag
$$}Since $r_{{d}^{1}_i} = r_{{d}^{1}_j}$ and $\mathcal{R}_{i,j}$ is bijective, we also have:
{\small \begin{align}
    & \prod_{d_i \in \mathbf{D}_i} F_{q_{d_i}}(r_{d_i}) = \prod_{{d}^{1}_i \in \mathbf{D}^{1}_i} \Big(F_{q_{{d}^{1}_i}}(r_{{d}^{1}_i}) \cdot \prod_{h \in \mathbf{D}_{{d}^{1}_i}} F_{q_h}(r_h)\Big) \notag \\
      =& \prod_{{d}^{1}_j \in \mathbf{D}^{1}_j} \Big(F_{q_{{d}^{1}_j}}(r_{{d}^{1}_j}) \cdot \prod_{k \in \mathbf{D}_{{d}^{1}_j}} F_{q_k}(r_k) \Big)
     =  \prod_{d_j \in \mathbf{D}_j} F_{q_{d_j}}(r_{d_j}). \notag
\end{align}}\\
This yields $p_i^{lom} = p_j^{lom}$, which finally gives us $r_i = r_j$ by Lemma \ref{app_lm:identical threshold}.
With the same process we can inductively determine the  type of each agent in the ranking tree, whereby the following theorem is obtained.
\begin{figure*}[ht]
    \centering
    \includegraphics[width=0.95\textwidth]{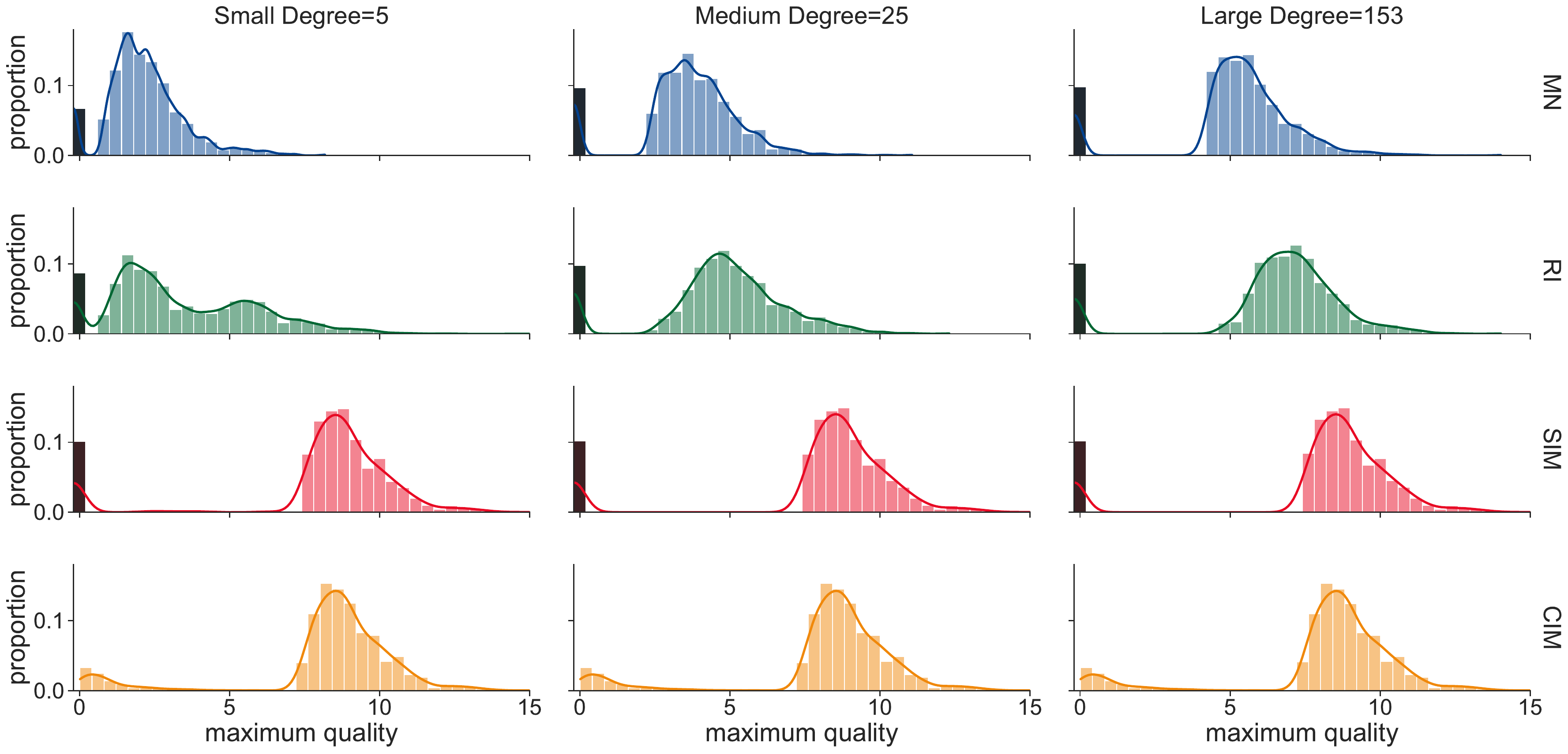}
    \caption{Equilibrium performances of  invitation contests mechanisms SIM and CIM over $1000$ different tasks. We compare them  with traditional contest mechanism with no invitation (MN), as well as a naive random invitation contest (RI) where each agents has a probability of $10\%$ to invite all her neighbors. Both SIM and CIM provide strong incentive for agents to invite, while MN and RI do not provide incentive  for invitation.}
    \label{fig:max quality}
\end{figure*}

\begin{theorem} \label{app_ob:type threshold}
 In the Bayesian Nash equilibrium of CIM, agents of the same type have an identical threshold. The threshold of every type is computed via a bottom-up recursion on the ranking tree.
\end{theorem}



Theorem \ref{app_ob:type threshold} shows the property of contribution behaviors of the agent types.
With Algorithm \ref{app_alg:solving CIM} and Theorem \ref{app_ob:type threshold}, we have the following corollary, which reveals the property of invitation behaviors of the agent types.

\begin{corollary} \label{ob:type deviate incentive}
    Given all agents fully invite all their neighbors, if agent $i$  has no incentive to deviate from ``fully inviting'', then every agent of $i$'s type  has no incentive to deviate from ``fully inviting''.
\end{corollary}

We illustrate this corollary with a detailed example in Appendix \ref{app_verifying CIM invitation}.
With Theorem \ref{app_ob:type threshold} and Corollary \ref{ob:type deviate incentive}, we can further simplify Algorithm \ref{app_alg:solving CIM} by only calculating threshold of one agent of each type, rather than all agent of that type. This significantly reduces the computation complexity, making it possible to compute BNE for large graphs.


\section{Experiments}

We implement our contest mechanisms on a well-known Facebook network dataset with $4039$ nodes and $88234$ edges \cite{mcauley2012learning}. 
We randomly chose three nodes as principals. Their degrees are $5, 25$, and $153$, respectively, representing that they have different levels of social resources.
On this network, we implement four different contest mechanisms, including SIM, CIM, a traditional contest mechanism with no invitations (abbreviated as MN), as well as a naive random invitation contest in which the probability for any agent to invite her neighbors is $0.1$ (abbreviated as RI).

For each of the three principals and each of the four contest mechanisms, we simulate $1000$ different tasks. That is, there are totally $3\times4\times1000=12000$ contests simulated in our experiments. For each single task, nodes' abilities $q_i$ are sampled from the exponential distribution $E(\lambda=1)$.
This sampling well captures that people's abilities for different tasks may show a large diversity in crowdsourcing while people with high abilities are very rare. With the simulation data, we evaluate the performances of SIM and CIM in the following key criteria: the overall performance in Fig. \ref{fig:max quality}, the crowd coverage and the budget feasibility in Table \ref{app_tab:simulation statistics}, as well as the dynamics of the highest solution quality in Fig. \ref{app_fig:CIM dynamics}.



\subsection{Overall Performance and Comparison}
In Fig. \ref{fig:max quality}, each column is for a fixed principal with a certain degree; each row is for a contest mechanism, and from the top row to the bottom row, we have MN, RI, SIM, and CIM, respectively. Each subfigure is the statistical result of $1000$ different tasks.
The $x$-axis is the bins for the highest quality in each task. The $y$-axis is the proportion of the bins in $1000$ tasks. Each colored chart depicts the distribution of the best solution within $1000$ tasks. 

Note that the random invitation contest RI is an analog to the real-world invitation among people, which works as a naive baseline for verifying SIM and CIM. 
Note that, RI's invitation probability  $0.1$ is high enough compared with people's invitation probability in a real contest.




\begin{enumerate}[leftmargin=*]
        \item The best solutions under SIM and CIM are much better than those under MN and RI. It can be seen by comparing each row in a  column. MN and RI can never surpass SIM or CIM, even for the principal with a very large degree. This is because SIM and CIM provide strong incentives for inviting and contributing.
    \item Both SIM and CIM are robust against the principals' social resource diversity. This can be seen by comparing the columns in the third and fourth rows. The best solution distribution does not change too much for different principals. This robustness makes SIM and CIM not only work well for companies or big organizations, but also have advantages for principals who have limited social resources. This is especially suitable for websites like Github or StackExchange where most principals are normal users with small degrees.
    \item In all the subfigures, the black bars above $x=0$ depict the chances that no agent contributes under this mechanism. We can see that MN, RI, and SIM all suffer from these failure chances. In sharp contrast, CIM completely solves this problem.  
\end{enumerate}

\subsection{Crowd Coverage and Budget Feasibility}
Recall that according to the theoretical result about SIM (see Theorem \ref{corollary:SIM_critical_number} and Remark \ref{remark:sim}), in BNE, low-ability agents try their best to invite while high-ability agents contribute but do not invite. Thus it is possible that not all agents are invited into the contest under SIM. To inspect to what extent agents in a crowd can be invited, we simulate $1000$ tasks using SIM on the Facebook dataset. It is observed that in almost all $1000$ contests, all agents in the crowd are invited.
Among these contests, the ratios that all agents are invited range from $98.8\%$ to $99.4\%$, depending on the principal's degree. The detailed statistics are in the middle box of Table \ref{app_tab:simulation statistics}. For example, the item $4038:\mathbf{98.8\%}$ means that in $98.8\%$ of the $1000$ tasks, all the $4038$ agents are invited. 

Regarding the crowd coverage under CIM, it is even more satisfactory,  in all the experiments under CIM, we find that \emph{every agent always fully invites all her neighbors}. More importantly, from all $1000$ contests under CIM, we find that when all the other agents invite all their neighbors, ``inviting all her neighbors'' is $i$'s best response. This indicates that, \textit{every agent invites all her neighbors} should be a  Bayesian Nash equilibrium of CIM. This is a strong evidence for that CIM performs even better than what Theorem \ref{app_alg:solving CIM} asserts. It provides a very strong incentive for invitation.

Regarding the budget, since SIM only selects out one winner, it is always budget feasible. For CIM, in $98.6\%$ of all the $1000$ tasks, \textit{only one winner} is selected. This indicates that, although CIM may select out multiple winners and award each of them, \textit{such multi-winner outcomes are very rare}.
Specifically, even when there is more than one winner, in all the experiments of CIM, the principal only needs to award \emph{at most two} winners. These results are in the right box of Table \ref{app_tab:simulation statistics}. For example, the item $1:\mathbf{98.8\%}$ means among all $1000$ contests under CIM, $988$ contests select out only $1$ winner. Moreover, the item $2:\mathbf{1.2\%}$ means all the remaining $12$ contests select out only $2$ winners. This means CIM not only guarantees high-quality solutions; it is also nearly budget-feasible since in expectation, the principal's expenditure is very close to her budget for one prize.

\begin{table*}[t]
    \centering
       \begin{tabular}{|c|lllll|ll@{}|}
            \hline 
            \makecell[c]{principal\\degree} & \multicolumn{5}{c|}{\makecell[c]{SIM: invited agent number \\and ratio in $1000$ tasks}} & \multicolumn{2}{c|}{\makecell[c]{CIM: winner number\\and ratio in $1000$ tasks}} \\ [8pt] \hline 
            $5$      & $4038:\mathbf{98.8\%}$  & $5:1.1\%$    & $ 4015:0.1\%$  &         &                      & $1: \mathbf{98.8\%}$             & $2: 1.2\% $                            \\ [6pt]
            $25$     & $4038:\mathbf{99.3\%}$  & $4036:0.2\%$  & $4035:0.2\%$  & $4027: 0.2\%$  & $4006:0.1\%$                 &$1: \mathbf{98.9\%}$            & $2: 1.1\% $                            \\[6pt]
            $153$    & $4038:\mathbf{99.4\%}$  & $4035: 0.1\%$  & $4031:0.1\%$  & $4027: 0.1\%$  &  $3979\sim 4023:0.3\%$    & $1: \mathbf{98.6\%}$       & $2: 1.4\% $                            \\ \hline
        \end{tabular}
    \caption{Crowd coverage of SIM and winner number of CIM. }
    \label{app_tab:simulation statistics}
\end{table*}

\subsection{Dynamics of Highest Quality}
In this part, we will reveal how the population size $n$ of the crowd $\mathbf{N}$ affects the highest quality of the solutions elicited by these mechanisms.
By the notion \emph{highest quality}, we mean the quality of the best solution among all the contributed solutions from agents.
Given a number $n$ of agents with i.i.d. abilities, we will first find the upper bound of the performance of \textit{any} crowdcoursing contest. Then, we will numerically study how the highest qualities in our mechanisms evolve as $n$ increases and compare them with this upper bound.

Consider  $n$ agents with an \emph{arbitrary network structure}. The highest ability among the $n$ agents is the largest order statistic $q_{(n)}$ whose cdf is 
{\small \begin{align}
\hat{F}_{q_{(n)}}(x) = \prod_{i \in [n]} P(q_i \leq x) = \prod_{i \in [n]} F_{q_{i}}(x) = F(x)^{n}.\notag
\end{align}}Thus the PDF of $q_{(n)}$ is 
$
    \hat{f}(x) = \frac{\mathrm{d} \hat{F}(x)}{\mathrm{d} x}  = n\cdot F(x)^{n-1} \cdot f(x). \notag
$
And the expected value of $q_{(n)}$ is 
{\small \begin{equation}
    \label{app_eq:best solution}
    E(q_{(n)}) = \int_l^u x \cdot \hat{f}(x) \mathrm{d}x = \int_l^u n \cdot x \cdot F(x)^{n-1} \cdot f(x) \mathrm{d}x.
\end{equation}}This is the expectation of the highest ability in a crowd of $n$ agents. It is also the upperbound of the feasible highest quality  in any crowdsourcing contest.
\emph{Traditional crowdsourcing contests cannot reach this upperbound}, simply because the number of agents in them cannot reach the value $n$.  



\begin{figure}[bh]
    \centering
 \subfigure[SIM]
 {\includegraphics[width=.495\textwidth]{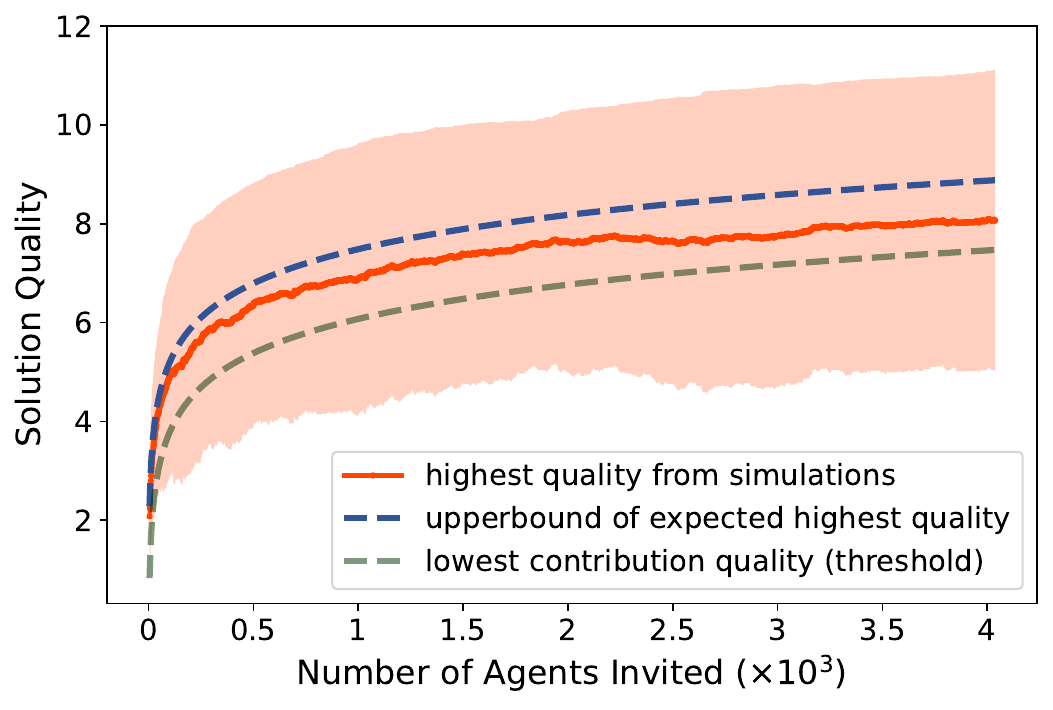}}
    \subfigure[CIM]    
 {\includegraphics[width=.495\textwidth]{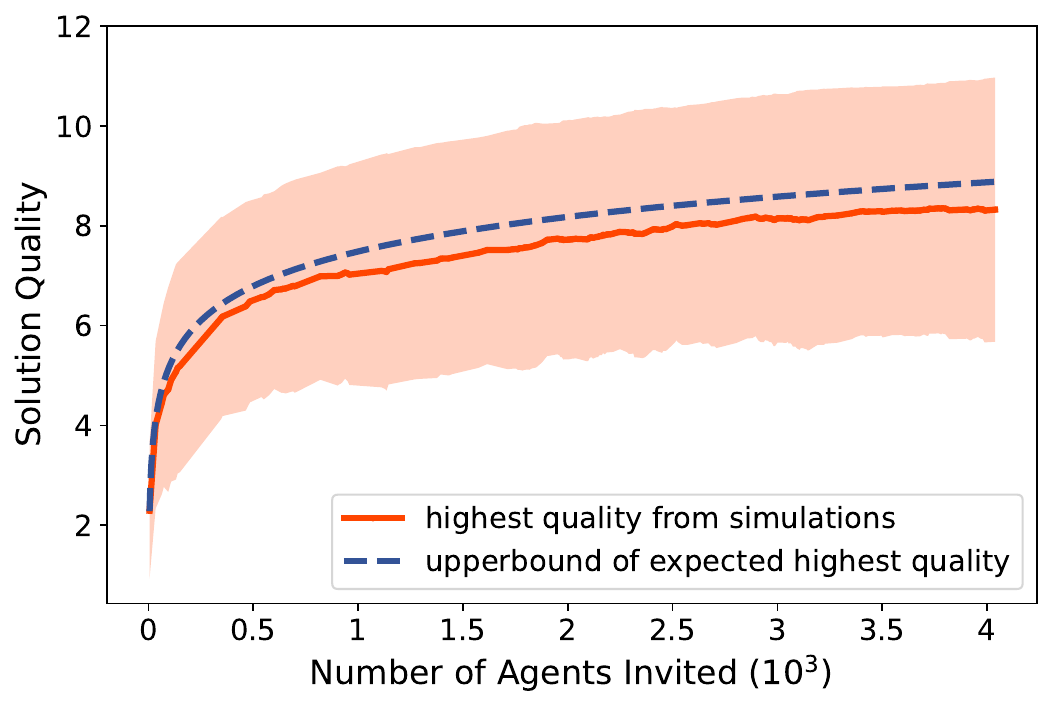}}
        \caption{Upper-bound and  dynamics of highest qualities of SIM and CIM.}
        \label{app_fig:CIM dynamics}
\end{figure}

To verify the worst-case performance of SIM and CIM, we choose a node with very small degree ($d=5$) as the  principal. If there is no  invitations (i.e., by traditional contests), the task posted by the principal can only acquire $5$ participants.
We sample the abilities of agents in the same Facebook network for $500$ times according to the same distribution. This means we totally simulate $500$ different crowdsourcing tasks in the same social network.

The results are shown in Fig. \ref{app_fig:CIM dynamics}, where the maximum number of agents is $4038$. The solid red curve represents the average highest qualities, while the red shadow alongside this solid curve is the standard deviation. The blue dashed line is the upper bound of the highest qualities of any crowdsourcing contest, which is calculated by  Eq.\eqref{app_eq:best solution}.
Under SIM and CIM, the participant population expands gradually. At the very beginning of the contest, only the principal's $5$ neighbors participate and the highest quality is low. Then it increases rapidly before $n=1000$. After that, it grows steadily with the increasing population size.

Both invitation mechanisms show excellent performances. The expectation of the highest quality they provide is close to the upper bound of any mechanism's feasible quality. Moreover,
CIM performs even better and is more stable than SIM. We do not show the performances of the existing crowdsourcing contest models since they do not have incentive design, and their population sizes can not be guaranteed in BNE. In most cases, their participant populations are limited to a fixed small group.

\section{Conclusions}


Existing crowdsourcing contest models only consider the agents who directly link to the principal, while a large amount of agents in the social network who do not know the task could not participate.
We extend contest design theory to social network environments and introduce new contest mechanisms whereby agents are impelled to invite their friends.
Our new invitation model and contest mechanisms capture the commonly seen scenarios where an agent can exert influence on the crowdsourced task not only by her own ability but also via her social influence. In contrast to traditional crowdsourcing contest methods, our new mechanisms can discover high-ability agents who hide deeply in the social network, which significantly improves the crowdsourcing quality. Our work is the first attempt to identify how social network structure affect agents' incentives in the \emph{equilibrium} of a crowdsourced task.
Future research could be the extensions of the invitation contests under complex settings, such as sequential contests, contests with uncertain performance, and so on. Another future direction is to investigate the change in the principal's profit when more agents are invited into the contest.  
Moreover, applying invitation contest mechanisms to various scenarios, such as privacy-preserving crowdsourcing \cite{wang2021privacy} and agent learning  crowdsourcing \cite{wang2021energy}, is also good topics. 

 \ifCLASSOPTIONcaptionsoff
  \newpage
\fi

\bibliographystyle{IEEEtran}
\bibliography{IEEEabrv,references}

%
%

\begin{IEEEbiography}[{\includegraphics[width=1in,height=1.25in,clip,keepaspectratio]{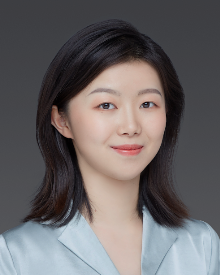}}]{Qi Shi}
received a Bachelor's degree in Management from the Shanghai University of Finance and Economics, China, in 2019 and a Master's degree in Computer Science from the University of Electronic Science and Technology of China in 2022.
She is pursuing a Ph.D. degree in computer science.
Her current research topics include multi-agent systems, algorithmic game theory, mechanism design, and their extensions to social networks.
\end{IEEEbiography}
\begin{IEEEbiography}[{\includegraphics[width=1in,height=1.25in,clip,keepaspectratio]{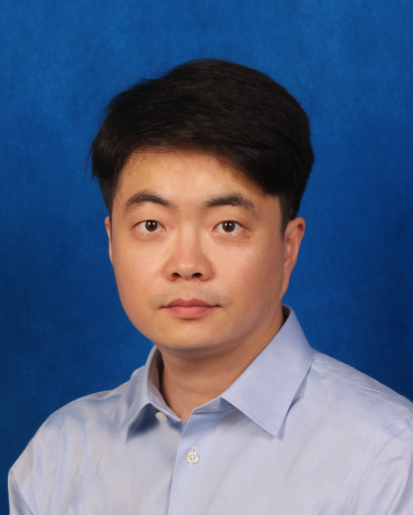}}]{Dong Hao}
is currently an Associate Professor at the University of Electronic Science and Technology of China. In 2013, he received the Ph.D. degree from the School of Informatics, Kyushu University, Japan. His research interests are at the intersection of artificial intelligence, game theory, and economics. His research topics include computational game theory, mechanism design, social choice, auctions, and networks. His publications appear in Artificial Intelligence Journal, IEEE/ACM Transactions on Networking, Computer Communications, AAAI, IJCAI, AAMAS, etc. He has been serving as reviewer for  journals such as AIJ, JAAMAS, WWW and several IEEE Transactions. He also serves as PC member or Senior PC member for conferences such as AAAI, IJCAI, AAMAS, ECAI, etc.
\end{IEEEbiography}

\newpage

\appendices

\section{Invitation  Contest Model}
\subsection{Leading Relation and Leader Set}
The invitation graph is a connected graph constructed with invited agents to be the nodes and invitations to be the edges. As Definition \ref{df:critical referrer} shows, if an agent $j$ is led by another agent $i$, then every path from the principal  $p$ to the agent node $j$ passes the agent node $i$ in the invitation graph.

The leading relation is asymmetric and transitive. That is, if $i$ leads $j$, then $j$ does not lead $i$; if $i$ leads $j$ and $j$ leads $k$, then $i$ leads $k$. More formally, when $i$ leads $j$, we have $i \in \mathbf{C}_j$, $j \in \mathbf{D}_i$ and $\mathbf{D}_j \subset \mathbf{D}_i$. Moreover, we have the following proposition about leading relation.

\begin{proposition}
\label{app_critical list}
  For an arbitrary agent $i$ with more than one leader (i.e. $|\mathbf{C}_i| > 1$), for any $j, k \in \mathbf{C}_i(j \neq k)$, there must be either $j$ leads $k$ or $k$ leads $j$ (i.e. $j \in \mathbf{C}_k \text{ or } k \in \mathbf{C}_j$).
\end{proposition}
\begin{proof}
    Because both $j$ and $k$ are $i$'s leaders, every path from the principal $p$ to $i$ passes both $j$ and $k$. According to the definition of leading, if $j$ does not lead $k$, there is at least one path $\rho_{pk}$ from $p$ to $k$ that doesn't pass $j$.
    The premise that $j$ leads $i$ requires that $j$ must be on every path from $p$ to $i$. Then every path from $k$ to $i$ must pass $j$. Otherwise, there will exist one path $\rho_{ki}$ not passing $j$, making a path $\rho_{pk} + \rho_{ki}$ from $p$ to $i$ that doesn't pass $j$. This contradicts the premise that $j$ is on every path from $p$ to $i$.

    {\begin{figure}[htb]
    \centering
		\begin{tikzpicture}[thick,global scale=0.8]
		    \node (1) at (-1,0.0)[circle,draw,minimum size=6mm]{};
    		\node (2) at (1,1)[circle,draw,minimum size=6mm]{}
    		edge[-, dashed] (1);
    		\node (3) at (1.0,-1)[circle,draw,minimum size=6mm]{}
    		edge[-, dashed, red] (1)
    		edge[-, dashed] (2);
    		\node (4) at (3,0)[circle,draw,minimum size=6mm]{}
    		edge[-, dashed, red] (2)
    		edge[-, dashed] (3);
    		
    		\node (11) at (-1,0.0){$p$};
    		\node (12) at (1,1){$k$};
    		\node (13) at (1,-1){$j$};
    		\node (14) at (3,0){$i$};
    		
    		\node (21) at (0,-0.8)[text=red]{\cancel{$\rho_{pj}$}};
    		\node (22) at (2,0.8)[text=red]{\cancel{$\rho_{ki}$}};
    		\node (23) at (0,0.8){$\rho_{pk}$};
    		\node (24) at (0.7,0){$\rho_{kj}$};
    		\node (25) at (2,-0.8){$\rho_{ji}$};
		\end{tikzpicture}
        \caption{A part of an invitation graph, where every dashed edge represents the paths between two nodes and excluding the other two nodes in this graph. If $j,k \in \mathbf{C}_i$, and $j$ does not lead $k$, then neither path $\rho_{pj}$ that does not pass $k$, nor path $\rho_{ki}$ that does not pass $j$, exists in the invitation graph.}
        \label{app_fig:trasitivity in domination relation}
    \end{figure}
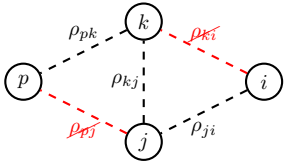 }

    Now we know that, if $j$ does not lead $k$, every path from $k$ to $i$ must pass $j$. This also means every path $\rho_{ji}$ from $j$ to $i$ doesn't pass $k$.
    Suppose there is a path $\rho_{pj}$ from $p$ to $j$ that doesn't pass $k$. Then there must be a path $\rho_{pj} + \rho_{ji}$ from $p$ to $i$ that doesn't pass $k$, which contradicts the premise that $k$ leads $i$. Hence, if $j$ does not lead $k$, every path from $p$ to $j$ must pass $k$, which means $k$ leads $j$.
\end{proof}

From Proposition \ref{app_critical list}, we know that for any agent $i$ with $\mathbf{C}_i \neq \emptyset$, there's a leader sequence  ($c_i^1$, $c_i^2$, $\dotsc$, $c_i^{|\mathbf{C}_i|}$) between $p$ and $i$, where every one in this sequence leads all her latter ones, which can be notated as:
$c_i^j \in \mathbf{C}_{c_i^{k}}$ for any $1 \leq j < k \leq |\mathbf{C}_i|$.
Also, we know that $\forall 1 < j \leq |\mathbf{C}_i|$, ($c_i^1$, $\dotsc$, $c_i^{j-1}$) is the leader sequence of $j$.

\section{CIM Equilibrium Analysis}

\subsection{Unconditional Contributors} \label{app_unconditional contributor}

As is discussed in the proof of Proposition \ref{theorem:CIM contributing strategy}, there are two cases for an arbitrary agent $i$ in equilibrium: (1) $\pi_i(l)\le0$; (2) $\pi_i(l)>0$. In the first case, $i$ has an equilibrium threshold $r_i \in [l,u)$, where $\pi_i(q_i=r_i) = 0$, and $i$ contributes only when $q_i>r_i$. In the second condition, we artificially set $r_i=l$, which does not change the analysis, because
under this condition, $i$ will always contribute, and
\begin{itemize}
    \item from $i$'s perspective, the probability that she beats any other $j$ is $P(q_j < q_i \lor q_j < r_j) = F_{q_j}(\max \{q_i, r_j\})$, which is not affected by the value of $r_i$;
    \item from any other $j$'s perspective, the probability that $j$ beats $i$ is $P(q_i < q_j) = F_{q_i}(q_j)$. Since $q_j \geq l$, this value exactly equals $F_{q_i}(\max \{q_j, r_i\})$.
\end{itemize}

Therefore, setting $r_i=l$ exerts the same influence on every agent, and thus does not affect the equilibrium of the game.
It is worth noting that, for such agent $i$ whose $\pi_i(l)>0$, under the setting that $r_i=l$, $\pi_i(r_i)>0$, which is different from the other kind of agent $j$, for whom $\pi_{j}(r_j) = 0$.
Given the above facts, we conclude that, for an arbitrary agent $i$, \begin{itemize}
    \item if $r_i>l$, then $\pi_i(r_i) = 0$;
    \item if $r_i=l$, then $\pi_i(r_i) \ge 0$.
\end{itemize}
$i$ might be an unconditional contributor for whom $\pi_i(r_i) > 0$ only if $r_i=l$. We have one thing for sure that, such an $i$ has the lowest threshold among all agents in equilibrium. Now we illustrate a lemma to see some properties about the lowest threshold.

\begin{lemma}\label{app_lm:minimum threshold}
If there is a unique agent $i$ who has the lowest threshold among all agents, then this lowest threshold $r_{min}=r_i=l$.
If there are multiple agents $i, \cdots, j$ who have the lowest thresholds among all agents, then the lowest threshold $r_{min} = r_i = \cdots = r_j > l$.
\end{lemma}

\begin{proof}
    When there is a unique agent $i$ with $r_{min}$ in equilibrium, denote the second lowest threshold by $r_{se}$. We have that $\forall j \ne i, r_j \ge r_{se} > r_{min}$. On the one hand, when $i$ contributes with $q_i = r_{se} > r_{min}$, her expected utility should be positive, that is
    {\small \begin{equation}
        \pi_i(r_{se}) = \mathcal{M} \cdot \prod_{j \in \mathbf{P}_i} F_{q_j}(\max\{r_j, r_{se}\}) - c > 0. \notag
    \end{equation}}On the other hand, when $i$ contributes with $q_i = l \leq r_{min} < r_{se} \leq r_{j \neq i}$, the expected utility is {\small \begin{align}
        \pi_i(l) & = \mathcal{M} \cdot \prod_{j \in \mathbf{P}_i} F_{q_j}(\max\{r_j, l\}) - c \notag \\
        & = \mathcal{M} \cdot \prod_{j \in \mathbf{P}_i} F_{q_j}(r_j) - c \notag \\
        & = \mathcal{M} \cdot \prod_{j \in \mathbf{P}_i} F_{q_j}(\max\{r_j, r_{se}\}) - c \notag \\
        & = \pi_i(r_{se})>0. \notag
    \end{align}}This indicates that $i$ is an unconditional contributor who will contribute with any ability and for whom $r_i=l$. Every other agent is not an unconditional contributor since they have thresholds higher than $l$.

    When there are multiple agents $i, \cdots, j$ with $r_{min}$ in equilibrium, $r_i = r_j = r_{min}$. Now we consider the leading relation between $i$ and $j$. We know $\mathbf{U} = \{i\} + \mathbf{D}_i + \mathbf{P}_i$. Thus $j \in \mathbf{D_i}$ or $j \in \mathbf{P}_i$:\begin{itemize}
        \item if $j \in \mathbf{D}_i$, which means, $i$ leads $j$, then $i \in \mathbf{P}_j$;
        \item if $j \notin \mathbf{D}_i$, then $j \in \mathbf{P}_i$.
    \end{itemize}

    We summarize the above as $i \in \mathbf{P}_j$ or $j \in \mathbf{P}_i$. Let's suppose $j \in \mathbf{P}_i$.
    The expected utility for $i$ to contribute with $q_i =r_i=r_{min}$ is
    {\small \begin{align}
        \pi_i(r_{i}) = & \mathcal{M} \cdot F_{q_j}(\max \{r_{j}, r_i\}) \cdot \prod_{k \in \mathbf{P}_i, k \ne j}F_{q_k}(\max \{r_k, r_i\})  - c \notag \\
        = & \mathcal{M} \cdot F_{q_j}(r_{min}) \cdot \prod_{k \in \mathbf{P}_i, k\ne j}F_{q_k}(r_k)  - c. 
    \end{align}}If $r_{min} = l$, $F_{q_j}(r_{min}) = 0$, and $\pi_i(r_i) = \mathcal{M} \cdot 0 - c < 0$, which contradicts the premise that $r_i=r_{min}$ is the threshold of $i$. Therefore, $r_{min} > l$ when there are multiple agents with $r_{min}$. And there is no unconditional contributor.
\end{proof}

Lemma \ref{app_lm:minimum threshold} gives a clear illustration about the unconditional contributors: (1) there is at most one unconditional contributor in each contest; (2) only for this unconditional contributor (denote by $i$), $r_i=l$ and $\pi_i(r_i)>0$; (3) for every other agent $j$, $r_j>l$ and $\pi_j(r_j)=0$;
(4) the unconditional contributor does not necessarily exist in every contest. Within one contest under CIM, for an arbitrary agent $i$, \begin{itemize}
    \item $r_i>l$ if and only if $\pi_i(r_i)=0$;
    \item $r_i=l$ if and only if $\pi_i(r_i)>0$, and such an $i$ is the unique unconditional contributor in this contest.
\end{itemize}

Note that, for the unconditional contributor $i$, all other agents have higher thresholds than her, which is exactly the condition that $|\mathbf{Q}_i|=0$ in Theorem \ref{theorem:threshold for non-leaf agent}.

\subsection{Proof of Proposition \ref{app_lm:leaf type threshold}}\label{appendix: leaf_node_threshould_proof}
{\textbf{Proposition \ref{app_lm:leaf type threshold}.} \textit{In the equilibrium under CIM, every agent $z$ who is a leaf node in the ranking tree of the invitation graph has an identical threshold $r_{z} = F^{-1}\left(\sqrt[|\mathbf{U}|-1]{\frac{c}{\mathcal{M}}}\right)$, which is the highest among all agents' thresholds.}}
\begin{proof}
    Assume all leaf agents except $i$ use the threshold $r_{z}$ in equilibrium. If $i$ is also a leaf agent, when $i$ contributes with $q_i$, the expected utility is
    {\small \begin{align}
        \pi_{i}(q_i) = \mathcal{M} \cdot \prod_{\begin{subarray}{l}
        z \in \mathbf{Z}, z \neq i
        \end{subarray}} & F_{q_z}\left(\max\{r_{z}, q_i\}\right) \notag 
        \cdot  \prod_{k \notin \mathbf{Z}} F_{q_k}(\max\{r_k, q_i\}) - c. \notag
    \end{align}}Here $\mathbf{Z}$ is the set of all leaf agents. For any agent $k \notin \mathbf{Z}$, $|\mathbf{D}_k| \geq 1$ and $|\mathbf{P}_k| \le |\mathbf{U}|-2$. Thus according to Lemma \ref{app_lm:threshold upper bound} and that $F(x)$ is monotonically increasing, we know
    {\small \begin{align}
        r_k \leq F^{-1}\left( \sqrt[\leftroot{0}\uproot{8}|\mathbf{P}_k|]{\frac{c}{\mathcal{M}}}\right)
        < F^{-1}\left( \sqrt[\leftroot{0}\uproot{8}|\mathbf{U}|-1]{\frac{c}{\mathcal{M}}}\right) = r_{z}. \notag
    \end{align}}If $q_i = r_{z} - \delta$ where $\delta \rightarrow 0^+$, then $i$'s expected utility for contributing is
    {\small \begin{align}
        \pi_i(q_i) &= \mathcal{M} \cdot  \prod_{\begin{subarray}{l}
        z \in \mathbf{Z}, z \neq i
        \end{subarray}} F_{q_z}(r_{z})  \cdot \prod_{k \notin \mathbf{Z}}F_{q_k}(q_i) - c \notag \\
        &< \mathcal{M} \cdot  \prod_{\begin{subarray}{l}
        z \in \mathbf{Z}, z \neq i
        \end{subarray}} F_{q_z}(r_{z}) \cdot \prod_{k \notin \mathbf{Z}}F_{q_k}(r_{z})  - c \notag \\
        &= \mathcal{M} \cdot F(r_{z})^{|\mathbf{U}|-1} - c = 0.\notag
    \end{align}}This means $r_i > r_{z} - \delta$ where $\delta \rightarrow 0^+$. On the other hand, according to Lemma \ref{app_lm:threshold upper bound}, we have $r_i \leq r_{z}$. By sandwich rule, $r_i = r_{z}$.
\end{proof}

\section{Equilibrium Computation}
\subsection{Proof of Lemma \ref{app_lm:non dominating threshold relation}}\label{appendix:plomless} 
\textbf{Lemma \ref{app_lm:non dominating threshold relation}.} \textit{For any two agents $i$ and $j$ who don't have leading relation (i.e. $i \nsucc j  \land j \nsucc i$), if $p_i^{lom} < p_j^{lom}$, then $r_i < r_j$.}
\begin{proof}
    Because $i$ and $j$ don't have leading relation, we have $i \notin \mathbf{D}_j, j \notin \mathbf{D}_i$, and $\mathbf{D}_i \cap \mathbf{D}_j = \emptyset$. Denote the set of common competitors of $i$ and $j$ by $\mathbf{P}_{ij} = \mathbf{P}_i \cap \mathbf{P}_j$. Moreover, we know that $\mathbf{U} = \mathbf{D}_i + \mathbf{P}_i + \{i\} = \mathbf{D}_j + \mathbf{P}_j + \{j\}$, the relationships within these sets are shown in Fig. \ref{app_fig:set relations}. We have that $\mathbf{P}_{ij} = \mathbf{P}_{i} - \mathbf{D}_j - \{j\} = \mathbf{P}_{j} - \mathbf{D}_i - \{i\}$.
    \begin{figure}[htb]
        \centering
		\begin{tikzpicture}[thick]
	        \draw (0,0) rectangle (0.5,1);
	        \draw (0,1) rectangle (0.5,1.5);
	        \draw (0.5,0) rectangle (1.5,1.5);
	        \draw (1.5,0) rectangle (2,0.5);
	        \draw (1.5,0.5) rectangle (2,1.5);
	        \node (1) at (0.25,0.5) {$\mathbf{D}_i$};
	        \node (2) at (0.25,1.25) {$\{i\}$};
	        \node (3) at (1.0,0.75) {$\mathbf{P}_{ij}$};
	        \node (4) at (1.75,1.0) {$\mathbf{D}_j$};
	        \node (5) at (1.75,0.25) {$\{j\}$};
	        \node (6) at (1.0,1.7) {$\mathbf{U}$};
	
	        \draw (-1.5,0) rectangle (-0.5,1.5);
	        \draw (-2,0) rectangle (-1.5,1);
	        \draw (-2,1) rectangle (-1.5,1.5);
	        \node (1) at (-1.75,0.5) {$\mathbf{D}_i$};
	        \node (2) at (-1.75,1.25) {$\{i\}$};
	        \node (3) at (-1,0.75) {$\mathbf{P}_{ij}$};
	        \node (6) at (-1.25,1.7) {$\mathbf{P}_j$};
	
	        \draw (2.5,0) rectangle (3.5,1.5);
	        \draw (3.5,0) rectangle (4,0.5);
	        \draw (3.5,0.5) rectangle (4,1.5);
	        \node (1) at (3.75,1.0) {$\mathbf{D}_j$};
	        \node (2) at (3.75,0.25) {$\{j\}$};
	        \node (3) at (3,0.75) {$\mathbf{P}_{ij}$};
	        \node (6) at (3.25,1.7) {$\mathbf{P}_i$};
		\end{tikzpicture}
        \caption{Illustration of $i$ and $j$ who don't have leading relation.}
        \label{app_fig:set relations}
    \end{figure}
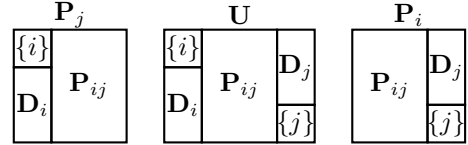

    Assume that $r_i \geq r_j$ in equilibrium, according to Lemma \ref{app_ps:threshold relation in dominating relation}, we have $\forall g \in \mathbf{D}_i,\ r_g > r_i \geq r_j$.
    According to Lemma \ref{app_lm:minimum threshold} in Appendix \ref{app_unconditional contributor}, when $r_i \geq r_j$, only $j$ could be the unconditional contributor. Therefore, the expected utility of $i$ when contributing with $q_i = r_i$ is
    {\small \begin{multline} \label{app_eq:unable to name 3}
      \pi_i(r_i) = \mathcal{M} \cdot \prod_{h \in \mathbf{D}_j} F_{q_h}(\max\{r_h, r_i\}) \\
        \cdot \prod_{k \in \mathbf{P}_{ij}} F_{q_k}(\max\{r_k, r_i\}) \cdot
        F_{q_j}(\max\{r_j, r_i\}) - c,
    \end{multline}} \\ 
    which equals $0$. On the other hand, the expected utility of $j$ when contributing with $q_j = r_j$ is
    {\small \begin{multline} \label{app_eq:U_j^P(r_j)}
        \pi_j(r_j) = \mathcal{M} \cdot  \prod_{g \in \mathbf{D}_i} F_{q_g}(\max\{r_g, r_j\}) \\
        \cdot  \prod_{k \in \mathbf{P}_{ij}} F_{q_k}(\max\{r_k, r_j\}) \cdot
        F_{q_i}(\max\{r_i, r_j\}) - c,
    \end{multline}}which is non-negative. Since $r_g > r_i \geq r_j$, $\forall g \in \mathbf{D}_i, F_{q_g}(\max\{r_g, r_j\}) = F_{q_g}(r_g)$, and under the i.i.d. assumption, $F_{q_i}(\max\{r_i, r_j\}) = F_{q_j}(\max\{r_i, r_j\}) = F(r_i)$. Combining Eqs.\eqref{app_eq:unable to name 3} and \eqref{app_eq:U_j^P(r_j)}, we have
    {\small \begin{multline} \label{app_eq:some inequility}
        \prod_{h \in \mathbf{D}_j} F_{q_h}(\max\{r_h, r_i\}) \cdot \prod_{k \in \mathbf{P}_{ij}} F_{q_k}(\max\{r_k, r_i\})\\
        \leq \prod_{g \in \mathbf{D}_i} F_{q_g}(r_g) \cdot \prod_{k \in \mathbf{P}_{ij}} F_{q_k}(\max\{r_k, r_j\}).
    \end{multline}}\\
    Since $r_i \geq r_j$, it is easy to know { $F_{q_k}(\max\{r_k, r_i\}) \geq F_{q_k}(\max\{r_k, r_j\}),$} thus {\small $$\prod_{k \in \mathbf{P}_{ij}} F_{q_k}(\max\{r_k, r_i\}) \geq \prod_{k \in \mathbf{P}_{ij}} F_{q_k}(\max\{r_k, r_j\}).$$}\\
    Combining it with Eq.\eqref{app_eq:some inequility}, we further have {\small $$\prod_{h \in \mathbf{D}_j} F_{q_h}(\max\{r_h, r_i\}) \leq \prod_{g \in \mathbf{D}_i} F_{q_g}(r_g).$$}\\
    Together with a simply known inequality that {$\prod_{h \in \mathbf{D}_j} F_{q_h}(r_h) \leq \prod_{h \in \mathbf{D}_j} F_{q_h}(\max\{r_h, r_i\}),$} we have
    {\small \begin{equation} \label{app_eq:monotonicity step}
         \prod_{h \in \mathbf{D}_j} F_{q_h}(r_h) \leq \prod_{g \in \mathbf{D}_i} F_{q_g}(r_g).
    \end{equation}}
    
    Note that, Eq.\eqref{app_eq:monotonicity step} is the conclusion under the assumption that $r_i \geq r_j$ for $i$ and $j$ who don't have leading relation, which means
    {\small $
        r_i \geq r_j \Rightarrow \prod_{g \in \mathbf{D}_i} F_{q_g}(r_g) \geq \prod_{h \in \mathbf{D}_j} F_{q_h}(r_h), \notag
    $}
    then by contrapositive we have
    {$
        \prod_{g \in \mathbf{D}_i} F_{q_g}(r_g) < \prod_{h \in \mathbf{D}_j} F_{q_h}(r_h) \Rightarrow r_i < r_j \notag
    $}, which according to the definition of $p_i^{lom}$ is equivalent to
        $p_i^{lom} < p_j^{lom} \Rightarrow r_i < r_j.$
\end{proof}

\subsection{Proof of Lemma \ref{app_lm:identical threshold}}\label{appendix:plomeq}
 
\textbf{Lemma \ref{app_lm:identical threshold}.}\textit{    If $
         p_i^{lom} = p_j^{lom}
    $, then $r_i = r_j$.}
 
\begin{proof}
    According to Lemma \ref{app_ps:threshold relation in dominating relation}, we know $\forall g \in \mathbf{D}_i, r_g > r_i$ and $\forall h \in \mathbf{D}_j, r_h > r_j$. If $i$ leads $j$, then $(\{j\} + \mathbf{D}_j) \subseteq \mathbf{D}_i$. This relation leads to
    {\small \begin{equation}
        \prod_{g \in \mathbf{D}_i}F_{q_g}(r_g) \leq \prod_{h \in \mathbf{D}_j}F_{q_h}(r_h) \cdot F_{q_j}(r_j) < \prod_{h \in \mathbf{D}_j}F_{q_h}(r_h).
    \end{equation}} \\
    By contrapositive, if $\prod_{g \in \mathbf{D}_i} F_{q_g}(r_g)= \prod_{h \in \mathbf{D}_j} F_{q_h}(r_h)$, there is no leading relation between $i$ and $j$.
    Recall that by definition, $p_i^{lom} = \prod_{g \in \mathbf{D}_i}F_{q_g}(r_g)$.
    This means, given $p_i^{lom} = p_j^{lom}$, $i$ and $j$ don't have leading relation.
    Then $\mathbf{P}_j = \mathbf{D}_i + \mathbf{P}_{ij} + \{i\}$ and $\mathbf{P}_i = \mathbf{D}_j + \mathbf{P}_{ij} + \{j\}$ as Fig. \ref{app_fig:set relations} shows.
    In equilibrium, the expected utility for $j$ to contribute with $q_j = r_i$ is
    {\small \begin{multline}
        \pi_j(r_i) = \mathcal{M} \cdot \prod_{g \in \mathbf{D}_i} F_{q_g}(\max\{r_g, r_i\}) \\
        \cdot \prod_{k \in \mathbf{P}_{ij}} F_{q_k}(\max\{r_k, r_i\}) \cdot F_{q_i}(\max\{r_i, r_i\})  - c. \notag
    \end{multline}}Since $r_g > r_i, \forall g \in \mathbf{D}_i$, $F_{q_g}(\max\{r_g, r_i\}) = F_{q_g}(r_g)$, and the right side of the former equality can be simplified as
    {\small \begin{multline}
        \mathcal{M} \cdot \prod_{g \in \mathbf{D}_i} F_{q_g}(r_g)
        \cdot \prod_{k \in \mathbf{P}_{ij}} F_{q_k}(\max\{r_k, r_i\}) \cdot F_{q_i}(r_i)  - c.\notag
    \end{multline}}\\
    As $p_i^{lom} = p_j^{lom}$, which is $\prod_{g \in \mathbf{D}_i} F_{q_g}(r_g) = \prod_{h \in \mathbf{D}_j} F_{q_h}(r_h)$, we further get
    {\small \begin{multline} \label{app_eq:unable to name 4}
        \pi_j(r_i) = \mathcal{M} \cdot \prod_{h \in \mathbf{D}_j} F_{q_h}(r_h) \cdot \prod_{k \in \mathbf{P}_{ij}} F_{q_k}(\max\{r_k, r_i\}) \cdot F_{q_i}(r_i) - c. 
    \end{multline}}
    
    In equilibrium, $i$'s expected utility of contributing with $q_i = r_i$ must be non-negative, that is
    {\small \begin{multline}
         \pi_i(r_i) =  \mathcal{M} \cdot \prod_{h \in \mathbf{D}_j} F_{q_h}(\max\{r_h, r_i\}) \\
         \cdot \prod_{k \in \mathbf{P}_{ij}} F_{q_k}(\max\{r_k, r_i\}) \cdot
          F_{q_j}(\max\{r_j, r_i\}) - c \geq 0.\notag
    \end{multline}} \\
    Since $F_{q_h}(r_h) \leq F_{q_h}(\max\{r_h, r_i\})$, it is easy to see $\prod_{h \in \mathbf{D}_j}F_{q_h}(r_h) \leq \prod_{h \in \mathbf{D}_j}F_{q_h}(\max\{r_h,r_i\}).$ Moreover, $F_{q_i}(r_i) \leq F_{q_j}(\max\{r_j, r_i\})$ under the i.i.d. assumption. Therefore, the above equation is reformulated into
    {\small \begin{align} \label{app_eq:unable to name 5}
         \pi_i(r_i) \geq\mathcal{M} \cdot \prod_{h \in \mathbf{D}_j} & F_{q_h}(r_h) 
         \cdot   \prod_{k \in \mathbf{P}_{ij}} F_{q_k}(\max\{r_k, r_i\}) \cdot F_{q_i}(r_i) - c.\notag
    \end{align}} \\
    The right-hand side is identical to that in Eq.\eqref{app_eq:unable to name 4}. Thus,
    {\begin{equation} \label{app_eq:unable to name 6}
        \pi_i(q_i=r_i) \geq \pi_j(q_j=r_i).
    \end{equation}}

    According to Lemma \ref{app_lm:minimum threshold} in the Appendix \ref{app_unconditional contributor}, if $i$ is an unconditional contributor, she is the unique agent who has the lowest threshold $l$, and $\pi_i(q_i=r_i)>0$. Under this condition, $r_j > r_i = l$, which leads to $\pi_j(q_j=r_i) < 0 < \pi_i(q_i=r_i)$. This satisfies the relation in Eq.\eqref{app_eq:unable to name 6}.
    If $i$ is not an unconditional contributor, then $\pi_j(q_j=r_i) \leq \pi_i(q_i=r_i) = 0$, so $r_j \geq r_i$. In conclusion, when $p_i^{lom} = p_j^{lom}$,  $r_j \geq r_i$.

    With the same process on reversed $i$ and $j$, we can derive that when $p_i^{lom} = p_j^{lom}$, there is that $r_i \geq r_j$. Therefore, there must be that $r_i = r_j$ if $p_i^{lom} = p_j^{lom}$.
\end{proof}

\subsection{Complexity of of Algorithm \ref{app_alg:solving CIM}}\label{Complexity analysis}
This algorithm has four components. 
\changreflcolor{black}
\begin{enumerate}[leftmargin=*]
    \item 	In the initialization,  the major computation is searching the cut vertices and constructing the leading tree. Searching the cut vertices can be done using the classic Tarjan's algorithm, whose complexity is linear in $n$. Computing the leaf nodes’ threshold is $O(n^2)$.
    \item 	The second component is for computing the probability $p_i^{lom}$ for each available $i$. For each $i$, computing $p_i^{lom}$ needs to traverse all agents in set ${\mathbf D}_i$. There are at most n agents in either $\mathbf{U}$ or ${\mathbf D}_i$. The complexity is $O(n^2)$. \label{alg_complexity_2}
    \item   	The third component is a sorting with   complexity  $O(n^2)$.
	\item   The fourth component compute $r_i$  for all nodes in temp set $\mathbf{T}$, which has $n$ agents at most. The computation of $r_i$  for each $i$ is linear in $n$. The complexity is $O(n^2)$. \label{alg_complexity_4}
\end{enumerate}

Components \ref{alg_complexity_2}) - \ref{alg_complexity_4}) repeat in the \textbf{while} loop until all $n$ nodes are handled. Therefore, the overall complexity of this algorithm is $O(n^3)$. 
\changreflcolor{blue}

\section{Supplements for Experiments} \label{app_experiment}

\subsection{Reducing Computation Complexity} \label{app_verifying CIM invitation}

{
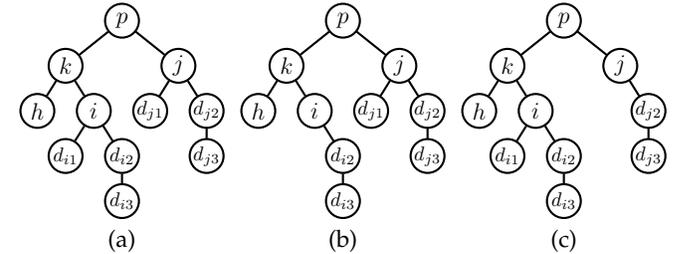
\begin{figure}[h]
    \centering
    \begin{subfigure}
        \centering
        \begin{tikzpicture}[thick, global scale=0.75]
            \node (p) at (0,0)[circle,draw,minimum size=6mm]{};
		    \node (1) at (-1,-0.8)[circle,draw,minimum size=6mm]{}
		    edge[-] (p);
    		\node (2) at (-1.5,-1.6)[circle,draw,minimum size=6mm]{}
    		edge[-] (1);
    		\node (i) at (-0.5,-1.6)[circle,draw,minimum size=6mm]{}
    		edge[-] (1);
    		\node (4) at (-0.0,-2.4)[circle,draw,minimum size=6mm]{}
    		edge[-] (i);
    		\node (5) at (-0.0,-3.2)[circle,draw,minimum size=6mm]{}
    		edge[-] (4);
    		\node (j) at (1.0,-0.8)[circle,draw,minimum size=6mm]{}
    		edge[-] (p);
    		\node (7) at (0.5,-1.6)[circle,draw,minimum size=6mm]{}
    		edge[-] (j);
    		\node (8) at (1.5,-1.6)[circle,draw,minimum size=6mm]{}
    		edge[-] (j);
    		\node (9) at (1.5,-2.4)[circle,draw,minimum size=6mm]{}
    		edge[-] (8);
    		\node (x) at (-1.0,-2.4)[circle,draw,minimum size=6mm]{}
    		edge[-] (i);
    		\node (10) at (0.0,0.0){\large $p$};
    		\node (11) at (-0.5,-1.6){\large $i$};
    		\node (12) at (1.0,-0.8){\large $j$};
    		\node (13) at (-1.0,-2.4){$d_{i1}$};
    		\node (14) at (0,-2.4){$d_{i2}$};
    		\node (15) at (0,-3.2){$d_{i3}$};
    		\node (16) at (0.5,-1.6){$d_{j1}$};
    		\node (17) at (1.5,-1.6){$d_{j2}$};
    		\node (18) at (1.5,-2.4){$d_{j3}$};
    		\node (k) at (-1, -0.8){\large $k$};
    		\node (h) at (-1.5, -1.6){\large $h$};
    		\node (20) at (0, -3.9){\large (a)};
		\end{tikzpicture}
    \end{subfigure}
    \begin{subfigure}
        \centering
        \begin{tikzpicture}[thick, global scale=0.75]
            \node (p) at (0,0)[circle,draw,minimum size=6mm]{};
		    \node (1) at (-1,-0.8)[circle,draw,minimum size=6mm]{}
		    edge[-] (p);
    		\node (2) at (-1.5,-1.6)[circle,draw,minimum size=6mm]{}
    		edge[-] (1);
    		\node (i) at (-0.5,-1.6)[circle,draw,minimum size=6mm]{}
    		edge[-] (1);
    		\node (4) at (-0.0,-2.4)[circle,draw,minimum size=6mm]{}
    		edge[-] (i);
    		\node (5) at (-0.0,-3.2)[circle,draw,minimum size=6mm]{}
    		edge[-] (4);
    		\node (j) at (1.0,-0.8)[circle,draw,minimum size=6mm]{}
    		edge[-] (p);
    		\node (7) at (0.5,-1.6)[circle,draw,minimum size=6mm]{}
    		edge[-] (j);
    		\node (8) at (1.5,-1.6)[circle,draw,minimum size=6mm]{}
    		edge[-] (j);
    		\node (9) at (1.5,-2.4)[circle,draw,minimum size=6mm]{}
    		edge[-] (8);
    		\node (10) at (0.0,0.0){\large $p$};
    		\node (11) at (-0.5,-1.6){\large $i$};
    		\node (12) at (1.0,-0.8){\large $j$};
    		\node (14) at (0,-2.4){$d_{i2}$};
    		\node (15) at (0,-3.2){$d_{i3}$};
    		\node (16) at (0.5,-1.6){$d_{j1}$};
    		\node (17) at (1.5,-1.6){$d_{j2}$};
    		\node (18) at (1.5,-2.4){$d_{j3}$};
			\node (k) at (-1, -0.8){\large $k$};
    		\node (h) at (-1.5, -1.6){\large $h$};
    		\node (20) at (0, -3.9){\large (b)};
		\end{tikzpicture}
    \end{subfigure}
    \begin{subfigure}
        \centering
        \begin{tikzpicture}[thick, global scale=0.75]
            \node (p) at (0,0)[circle,draw,minimum size=6mm]{};
		    \node (1) at (-1,-0.8)[circle,draw,minimum size=6mm]{}
		    edge[-] (p);
    		\node (2) at (-1.5,-1.6)[circle,draw,minimum size=6mm]{}
    		edge[-] (1);
    		\node (i) at (-0.5,-1.6)[circle,draw,minimum size=6mm]{}
    		edge[-] (1);
    		\node (4) at (-0.0,-2.4)[circle,draw,minimum size=6mm]{}
    		edge[-] (i);
    		\node (5) at (-0.0,-3.2)[circle,draw,minimum size=6mm]{}
    		edge[-] (4);
    		\node (j) at (1.0,-0.8)[circle,draw,minimum size=6mm]{}
    		edge[-] (p);
    		\node (8) at (1.5,-1.6)[circle,draw,minimum size=6mm]{}
    		edge[-] (j);
    		\node (9) at (1.5,-2.4)[circle,draw,minimum size=6mm]{}
    		edge[-] (8);
    		\node (x) at (-1.0,-2.4)[circle,draw,minimum size=6mm]{}
    		edge[-] (i);
    		\node (10) at (0.0,0.0){\large $p$};
    		\node (11) at (-0.5,-1.6){\large $i$};
    		\node (12) at (1.0,-0.8){\large $j$};
    		\node (13) at (-1.0,-2.4){$d_{i1}$};
    		\node (14) at (0,-2.4){$d_{i2}$};
    		\node (15) at (0,-3.2){$d_{i3}$};
    		\node (17) at (1.5,-1.6){$d_{j2}$};
    		\node (18) at (1.5,-2.4){$d_{j3}$};
			\node (k) at (-1, -0.8){\large $k$};
    		\node (h) at (-1.5, -1.6){\large $h$};
    		\node (20) at (0, -3.9){\large (c)};
        \end{tikzpicture}
    \end{subfigure}
    \caption{Three similar ranking trees. In (a), $i$ and $j$ are of the same type; in (b), $i$ doesn't invite $d_{i1}$; in (c), $j$ doesn't invite $d_{j1}$.}
    \label{app_fig:similar trees}
\end{figure}}

    Consider three similar ranking trees in Fig. \ref{app_fig:similar trees}. $i$ and $j$ in (a) are of the same type.
    Comparing with the case in (a), $i$ doesn't invite $d_{i1}$ in (b), and $j$ doesn't invite $d_{j1}$ in (c).

    With Lemmas \ref{app_lm:identical threshold} and \ref{app_lm:non dominating threshold relation}, it is easy to know the equilibrium thresholds $r_i = r_j$ in (a), $r_i' > r_j'$ in (b), and $r_i'' < r_j''$ in (c).
    It is worth noting that, if we run Algorithm \ref{app_alg:solving CIM} on both (b) and (c), the programs execute in a symmetric way until $r_i'$ is determined in (b), or $r_j''$ is determined in (c).

    To be specific, the agent numbers in (b) and (c) are equal. At the beginning of both programs, $h'$, $d_{i3}'$, $d_{j1}'$, $d_{j3}'$ in (b) and $h''$, $d_{i1}''$, $d_{i3}''$, $d_{j3}''$ in (c) are all leaf agents with identical threshold.
    Then in the first iteration of \textbf{while} loop of both programs, $d_{i2}'$, $d_{j2}'$ in (b) and $d_{i2}''$, $d_{j2}''$ in (c) are of the same type with identical threshold.
    In the next iteration, $i'$ in (b) and $j''$ in (c) are of the same type with identical threshold. That is to say, $r_i' = r_j''$.

    Considering Eq.\eqref{eq:CIM_contribute}, we know for any $q > r_i' =r_j''$, the expected utility for $i$ to contribute with $q$ in (b) equals the expected utility for $j$ to contribute with the same $q$ in (c) (i.e. $\pi_{i}'(q) = \pi_{j}''(q)$). And for any $q \leq r_i' = r_j''$, neither $i'$ nor $j''$ contributes.
    The former facts show that, if $i$ has no incentive to deviate from (a) to (b), then $j$ has no incentive to deviate from (a) to (c) either.

{
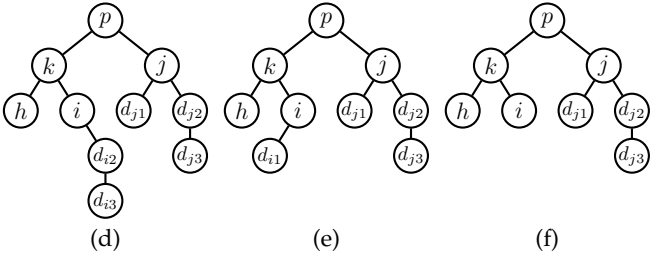
\begin{figure}[h]
    \centering
    \begin{subfigure}
        \centering
        \begin{tikzpicture}[thick, global scale=0.75]
            \node (p) at (0,0)[circle,draw,minimum size=6mm]{};
		    \node (1) at (-1,-0.8)[circle,draw,minimum size=6mm]{}
		    edge[-] (p);
    		\node (2) at (-1.5,-1.6)[circle,draw,minimum size=6mm]{}
    		edge[-] (1);
    		\node (i) at (-0.5,-1.6)[circle,draw,minimum size=6mm]{}
    		edge[-] (1);
    		\node (4) at (-0.0,-2.4)[circle,draw,minimum size=6mm]{}
    		edge[-] (i);
    		\node (5) at (-0.0,-3.2)[circle,draw,minimum size=6mm]{}
    		edge[-] (4);
    		\node (j) at (1.0,-0.8)[circle,draw,minimum size=6mm]{}
    		edge[-] (p);
    		\node (7) at (0.5,-1.6)[circle,draw,minimum size=6mm]{}
    		edge[-] (j);
    		\node (8) at (1.5,-1.6)[circle,draw,minimum size=6mm]{}
    		edge[-] (j);
    		\node (9) at (1.5,-2.4)[circle,draw,minimum size=6mm]{}
    		edge[-] (8);
    		\node (10) at (0.0,0.0){\large $p$};
    		\node (11) at (-0.5,-1.6){\large $i$};
    		\node (12) at (1.0,-0.8){\large $j$};
    		\node (14) at (0,-2.4){$d_{i2}$};
    		\node (15) at (0,-3.2){$d_{i3}$};
    		\node (16) at (0.5,-1.6){$d_{j1}$};
    		\node (17) at (1.5,-1.6){$d_{j2}$};
    		\node (18) at (1.5,-2.4){$d_{j3}$};
			\node (k) at (-1, -0.8){\large $k$};
    		\node (h) at (-1.5, -1.6){\large $h$};
    		\node (20) at (0, -3.9){\large (d)};
		\end{tikzpicture}
    \end{subfigure}
    \begin{subfigure}
        \centering
        \begin{tikzpicture}[thick, global scale=0.75]
            \node (p) at (0,0)[circle,draw,minimum size=6mm]{};
		    \node (1) at (-1,-0.8)[circle,draw,minimum size=6mm]{}
		    edge[-] (p);
    		\node (2) at (-1.5,-1.6)[circle,draw,minimum size=6mm]{}
    		edge[-] (1);
    		\node (i) at (-0.5,-1.6)[circle,draw,minimum size=6mm]{}
    		edge[-] (1);
    		\node (j) at (1.0,-0.8)[circle,draw,minimum size=6mm]{}
    		edge[-] (p);
    		\node (7) at (0.5,-1.6)[circle,draw,minimum size=6mm]{}
    		edge[-] (j);
    		\node (8) at (1.5,-1.6)[circle,draw,minimum size=6mm]{}
    		edge[-] (j);
    		\node (9) at (1.5,-2.4)[circle,draw,minimum size=6mm]{}
    		edge[-] (8);
    		\node (x) at (-1.0,-2.4)[circle,draw,minimum size=6mm]{}
    		edge[-] (i);
    		\node (10) at (0.0,0.0){\large $p$};
    		\node (11) at (-0.5,-1.6){\large $i$};
    		\node (12) at (1.0,-0.8){\large $j$};
    		\node (13) at (-1.0,-2.4){$d_{i1}$};
    		\node (16) at (0.5,-1.6){$d_{j1}$};
    		\node (17) at (1.5,-1.6){$d_{j2}$};
    		\node (18) at (1.5,-2.4){$d_{j3}$};
			\node (k) at (-1, -0.8){\large $k$};
    		\node (h) at (-1.5, -1.6){\large $h$};
    		\node (20) at (0, -3.9){\large (e)};
        \end{tikzpicture}
    \end{subfigure}
    \begin{subfigure}
        \centering
        \begin{tikzpicture}[thick, global scale=0.75]
            \node (p) at (0,0)[circle,draw,minimum size=6mm]{};
		    \node (1) at (-1,-0.8)[circle,draw,minimum size=6mm]{}
		    edge[-] (p);
    		\node (2) at (-1.5,-1.6)[circle,draw,minimum size=6mm]{}
    		edge[-] (1);
    		\node (i) at (-0.5,-1.6)[circle,draw,minimum size=6mm]{}
    		edge[-] (1);
    		\node (j) at (1.0,-0.8)[circle,draw,minimum size=6mm]{}
    		edge[-] (p);
    		\node (7) at (0.5,-1.6)[circle,draw,minimum size=6mm]{}
    		edge[-] (j);
    		\node (8) at (1.5,-1.6)[circle,draw,minimum size=6mm]{}
    		edge[-] (j);
    		\node (9) at (1.5,-2.4)[circle,draw,minimum size=6mm]{}
    		edge[-] (8);
    		\node (10) at (0.0,0.0){\large $p$};
    		\node (11) at (-0.5,-1.6){\large $i$};
    		\node (12) at (1.0,-0.8){\large $j$};
    		\node (16) at (0.5,-1.6){$d_{j1}$};
    		\node (17) at (1.5,-1.6){$d_{j2}$};
    		\node (18) at (1.5,-2.4){$d_{j3}$};
			\node (k) at (-1, -0.8){\large $k$};
    		\node (h) at (-1.5, -1.6){\large $h$};
    		\node (20) at (0, -3.9){\large (f)};
        \end{tikzpicture}
    \end{subfigure}
    \caption{Three  ranking trees if $i$ in Fig. \ref{app_fig:similar trees}(a) deviates from ``inviting all''.}
    \label{app_fig:deviation trees}
\end{figure}}

Take the case in Fig. \ref{app_fig:similar trees}(a) as an example to illustrate the above common changes.
 
    Suppose Fig. \ref{app_fig:similar trees}(a) is the ranking tree when all agents invite all their neighbors.
    Under this condition, $i$ leads $d_{i1}$, $d_{i2}$ and $d_{i3}$; the direct descendants are $d_{i1}$ and $d_{i2}$.

    There are three possible conditions if $i$ deviates from ``inviting all'': (1) $d_{i1}$ is excluded and $i$ leads $d_{i2}$ and $d_{i3}$ (as Fig. \ref{app_fig:deviation trees}(d) shows); (2) $d_{i2}$ is excluded, thus $d_{i3}$ also being excluded and $i$ leads $d_{i1}$ only (as Fig. \ref{app_fig:deviation trees}(e) shows); (3) $d_{i1}$, $d_{i2}$ and $d_{i3}$ are all excluded (as Fig. \ref{app_fig:deviation trees}(f) shows).

    We ran the equilibrium verification on the ranking tree in Fig. \ref{app_fig:similar trees}(a) by calculating the equilibrium thresholds on the ranking trees in Fig. \ref{app_fig:similar trees}(a) and Fig. \ref{app_fig:deviation trees}(d-f).
    Fig. \ref{app_fig:expected utility case}(a) shows these equilibrium thresholds when the prize $\mathcal{M} = 1$, the contributing cost $c = 0.1$, and agents' qualities are independently and identically distributed on the exponential distribution $E(\lambda=1)$.

    From Fig. \ref{app_fig:expected utility case}(a), we know that when $i$ invites all her neighbors, among all $i$'s competitors, $d_{j1}, d_{j2}, d_{j3}, h$ have higher thresholds than $i$; $k, j$ have thresholds equal to or lower than $i$. After $i$'s deviating, $i$'s threshold increases, thresholds of $d_{j1}, d_{j2}, d_{j3}, h$ decrease, and thresholds of $k, j$ are still equal lower than that of $i$.
 
\begin{figure}[hb]
    \centering
    \subfigure[Thresholds of $i$ and her competitors under 4 equilibria]{
    \centering
    \includegraphics[width=0.45\textwidth]{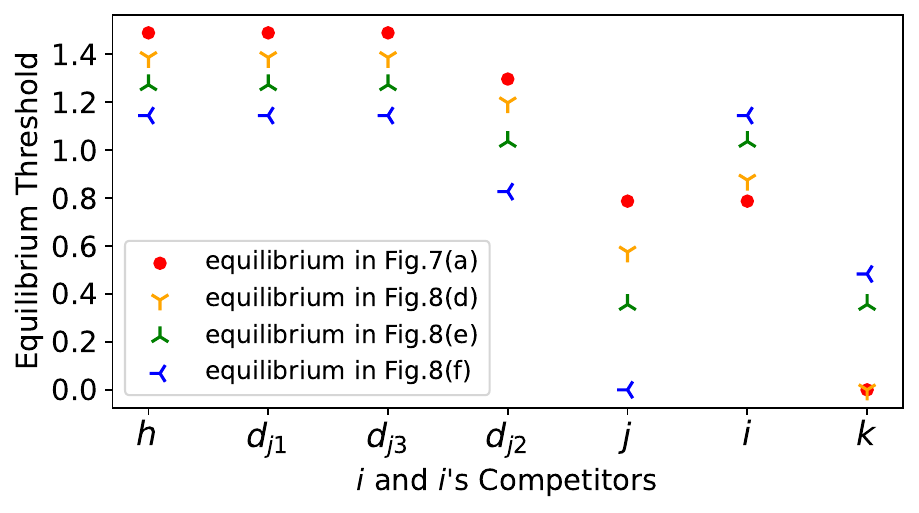}}
    \quad
    \subfigure[$i$'s expected utility under different intrinsic quality]{
    \centering
    \includegraphics[width=0.46\textwidth]{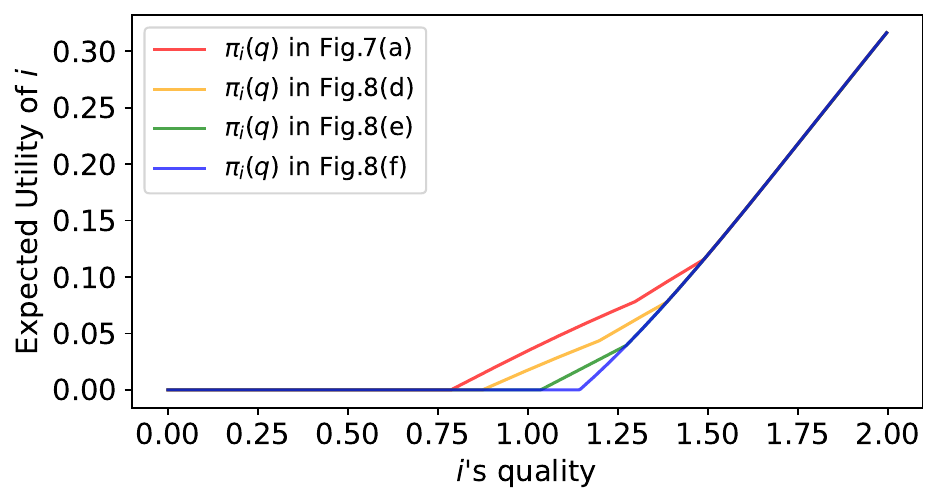}}
    \caption{In the case shown in Figs.\ref{app_fig:similar trees}(a), $i$'s different invitations may lead to different ranking trees shown in Fig. \ref{app_fig:deviation trees}(d-f). Sub-figures (a) and (b) illustrate the equilibrium information one these four ranking trees under the assumption that $\mathcal{M}=1, c=0.1$ and $q \sim E(\lambda=1)$.}
    \label{app_fig:expected utility case}
\end{figure}





\end{document}